\documentclass{article}

\usepackage[english]{babel}
\usepackage{cancel}
\usepackage{yhmath}
\usepackage{listings}
\usepackage{nicematrix}
\usepackage{float}
\usepackage{authblk}
\usepackage{makecell}
\usepackage{adjustbox} 
\usepackage{xcolor}
\usepackage[letterpaper,top=2cm,bottom=2cm,left=3cm,right=3cm,marginparwidth=1.75cm]{geometry}
\usepackage{amsmath}
\usepackage{amsthm}
\usepackage{multirow}
\usepackage{booktabs}
\usepackage{graphicx}
\usepackage[colorlinks=true, allcolors=blue]{hyperref}
\usepackage{amsfonts}
\usepackage{mdframed}

\newmdenv[
  topline=false,
  bottomline=false,
  rightline=false,
  linewidth=1pt,
  linecolor=black!20,
  skipabove=\topsep,
  skipbelow=\topsep,
  innerleftmargin=1em,
  innerrightmargin=0pt,
  innertopmargin=0pt,
  innerbottommargin=0pt
]{lemmapush}

\newcommand{\squaredtimes}{\mathbin{\scalebox{1}{\setlength{\fboxsep}{0.5pt}\setlength{\fboxrule}{0.5pt}\boxed{\circ}}}}

\newtheorem{theorem}{Theorem}[section]

\newtheorem{proposition}{Proposition}[section]

\newtheorem{corollary}{Corollary}[theorem]
\newtheorem{lemma}[theorem]{Lemma}
\newtheorem{observation}{Observation}[section]

\addto\captionsenglish{}
\newtheorem{remark}{Remark}

\DeclareMathOperator{\ReLU}{ReLU}
\DeclareMathOperator{\rank}{rank}
\DeclareMathOperator{\diag}{diag}

\newcommand\blfootnote[1]{%
  \begingroup
  \renewcommand\thefootnote{}\footnote{#1}%
  \addtocounter{footnote}{-1}%
  \endgroup
}

\title{$W_K, W_V$ is Probably All You Need: On the Necessity of the Query, Key, and Value Weight Triplet in Self-Attention Transformers}

\author[1]{Marko Karbevski\thanks{\texttt{marko.karbevski@gmail.com}, \texttt{marko.karbevski@insimplicity.tech}}}
\author[2]{Antonij Mijoski\thanks{\texttt{antonij.mijoski@unistra.fr}, 
\texttt{antonijmijo@gmail.com}}}
\affil[1]{In Simplicity Technologies}
\affil[2]{Institut de Recherche Mathématique Avancée (IRMA), 
Université de Strasbourg}

\begin{document}
\maketitle
\blfootnote{Code \& Checkpoints: \url{https://github.com/MarkoKarbevski/Wqkv_necessity}}
\blfootnote{To appear at the \href{https://delta-workshop.github.io/DeLTa2026/}{ICLR 2026 Workshop on Deep Generative Models (DeLTa)}.}
\begin{abstract}
We theoretically investigate whether the Query, Key, Value weight triplet can be reduced in encoder-only and decoder-only transformers. Under mild assumptions, we prove that one of the Query, Key or Value weights are redundant and can be replaced with the identity matrix, reducing attention parameters by 25\%. If applied to the Query or Key weights, this also simplifies optimization: attention logits depend on a single learned weight matrix rather than on a product of two. Validating the Query weight removal on decoder-only GPT-style small models trained from scratch, we find that reduced models match baseline performance despite fewer parameters, and outperform baselines when saved parameters are reallocated. Our analysis has also led us to a structural expressivity boundary: in the mathematically tractable ReLU setting, skip connections push MLPs into a generically disjoint function class at fixed width. These findings motivate investigation across modalities and at scale, where the observed stability and efficiency gains may prove most consequential.
\end{abstract}
\newpage 

\vspace*{\fill}
 \vspace*{\fill}
 \begin{center}
 \tableofcontents    
 \end{center}
 \vspace*{\fill}
\section{Introduction}

Training and deploying transformer-based language models~\cite{Attnisallyouneed} remains computationally expensive~\cite{samsi2023words}, motivating architectural optimizations~\cite{tay2022efficient} including quantization~\cite{xu2023parameter,MWM+24}, efficient attention~\cite{choromanski2020rethinking,wang2020linformer}, weight sharing~\cite{Shazeer2019FastTD,ALTdJ+23,lan2020albert}, and normalization streamlining and block restructuring~\cite{ICLR2024_simplyfing_transformer,heimersheim2024remove,baroni2025transformers,zhu2025transformersDyT}. Recent work has shown that normalization layers~\cite{heimersheim2024remove,baroni2025transformers} and attention parameters~\cite{ICLR2024_simplyfing_transformer} can be rearranged and simplified with minimal performance impact, suggesting current architectures may be overparameterized.

We investigate redundancy within the attention mechanism itself: \emph{is the entirety of the Query-Key-Value weight triplet necessary?} The key insight that motivated our research is the following: attention depends on the input $X$ only through the products $XW_Q$, $XW_K$, $XW_V$. This enables a telescoping construction where each layer's basis transformation prepares the input for the next, ultimately replacing $W_Q$ with the identity matrix throughout the network. Empirically, we validate that GPT-style models~\cite{Radford2019LanguageMA, brown2020language, Karpathy} trained from scratch with $W_Q = I_d$ achieve comparable validation loss to standard baselines, removing 25\% of attention parameters per layer (8\% of transformer block parameters). The deliberate choice of simplifying $W_Q=Id$ is made for consistency with the KV cache \& GQA optimization of LLMs, and empirically testing the simplification of $W_V$ or $W_K$ is left for future work, whereas the theoretical approach remains identical.

\subsection{Contributions}

We take a theory-first approach: we establish Query weight elimination under simplifying assumptions on normalization and skip connections, investigate the structural obstructions each introduces, and validate empirically in full architectural complexity.

\paragraph{Theoretical results.} We first observe that multi-head attention has intrinsic redundancy: any block-diagonal transformation can be absorbed into the Query-Key pair without changing the output (Proposition~\ref{prop:block-diag}). This redundancy is parametrized by a $(h \cdot d_k^2)$-dimensional manifold (one invertible $d_k \times d_k$ matrix per head), present in the full model with no simplifying assumptions. Under the additional assumption of no normalization layers, we prove Query weight redundancy in settings complementary to recent theoretical work by Graef~\cite{graef2024transformer}:

\begin{itemize}
\item \textbf{Single-layer ``Free Lunch'' (Theorem~\ref{thm:free-lunch}):} In any transformer without normalization, a single layer's Query weight can always be eliminated without architectural modifications, requiring only weight reparametrization. This applies to existing pretrained models once LayerNorm is removed via the techniques of Heimersheim~\cite{heimersheim2024remove} and Baroni et al.~\cite{baroni2025transformers}, providing a practical path to post-training Query weight elimination.

\item \textbf{Multi-layer Query weight elimination (Theorems~\ref{thm:AttentionSkipOnlyQWElim}, \ref{cor:weight-sharing}):} For transformers without layer normalization, we prove that $W_Q$ can be set to the identity in \emph{all} layers under two complementary conditions: (i) skip connections are placed exclusively around the attention sublayer, or (ii) weight-sharing is enforced across layers. The proofs rely on a characterization of how blocks transform between different bases.

\item \textbf{Semi-conjugacy of linear maps through layer normalization (Section~\ref{subsection:Layernorm_n_basischange}):} We derive sufficient conditions for basis transformations to commute with layer normalization (Lemma~\ref{lem:L_epsilon inversion lemma}, Theorem~\ref{cor:layernorm-basis}). We note that the resulting obstruction is strictly milder than the per-head modulation introduced by QK-normalization~\cite{henry2020querykey}, which poses no empirical limitation, as demonstrated by numerous SOTA models \cite{qwen2025, gemmateam2025gemma3technicalreport, olmo2025olmo3}.
\end{itemize}

To clarify our main observation regarding Multi-Head Attention, we introduce:
\begin{itemize}
    \item \textbf{Index-free notation for MHA (Section~\ref{subsection:block-notation}):} Building on variants of the Block Hadamard Products defined in \cite{Horn1991}, we provide a compact notation that more closely reflects standard implementations. Importantly, the motivating observation of our paper becomes mathematically trivial under this formulation.
\end{itemize}

Finally, the analysis of the above phenomena yields a fundamental result regarding the geometry of skip connections:
\begin{itemize}
\item \textbf{Solving the equation $\text{MLP} = \text{Id} + \text{MLP}$ (Theorem~\ref{prop:projector-characterization}):} We characterize when residual connections can be absorbed into ReLU MLPs used in transformers. We explicitly solve the functional equation $W_2\,\text{ReLU}(W_1 x) + x = V_2\,\text{ReLU}(V_1 x)$ for the unknown matrices $V_1, V_2$ given fixed $W_1, W_2$, where $V_i$ has the same dimensions as $W_i$ for $i \in\{1,2\} $.
\end{itemize}

\paragraph{Empirical validation.} We train GPT-style models (117M to 124M parameters) from scratch on OpenWebText, comparing against parameter-matched baselines on fully converged training runs ($\sim$29.5B tokens, $\sim$$12\times$ Chinchilla-optimal). With an adjusted attention scaling factor, the reduced 117M model matches the full 124M baseline despite 8\% fewer non-embedding parameters, while parameter-matched standard baselines at the same size perform measurably worse. Reallocating saved parameters to the MLP yields the best validation loss across all configurations. The reduced architecture benefits from higher learning rates, as attention logits depend on a single learned weight matrix rather than on a product of two (Section~\ref{sec:discussion}).

\paragraph{Scope and limitations.} This work establishes sufficient conditions for the elimination of Query weight matrices ($W_Q$). While the theoretical derivation utilizes simplified architectures, empirical validation on GPT-style models confirms that $W_Q$ redundancy persists in practical models. We prioritize evaluating learning capacity through log-loss performance. These experiments at the 117M-125M parameter scale demonstrate the technical feasibility of this simplification and establish the empirical baseline for systematic scaling, multi-seed validation, downstream benchmarking, and validation across modalities and architectures.

\paragraph{Organization.} We begin by positioning our work relative to recent simplification efforts in Section~\ref{sec:related}, then establish the essential notation and results that motivate our further analysis in Section~\ref{sec:prelim}. Section~\ref{sec:theory} forms the core of the paper: we present progressively stronger elimination results, from the single-head case through single-layer, attention-skip-only, and weight-shared multi-layer settings; Table~\ref{tab:coverage} summarizes which architectural features each result handles. We validate on GPT-style models trained from scratch with $W_Q = I_d$ in Section~\ref{sec:experiments}, and conclude with limitations and future directions in Section~\ref{sec:discussion}. The appendix contains further theoretical contributions: we characterize when basis transformations commute with LayerNorm (Appendix~\ref{subsection:Layernorm_n_basischange}), prove an exact condition for when residual connections can be absorbed into ReLU MLPs (Appendix~\ref{subsec:relu-skip}), and provide supporting experiments on MLP basis transfer approximation (Appendix~\ref{app:mlp}).

\begin{table}[!htb]
\centering
\label{tab:coverage}
\adjustbox{max width=\textwidth,center}{%
\begin{tabular}{|c|c|c|c|c|c|c|}
\hline
\textbf{Result Type} & \textbf{Section} & \textbf{Normalization} & \textbf{Skip Connections} & \textbf{Multihead} & \textbf{Multi-layer} & \textbf{Weight Decay} \\ 
\hline
\multirow{5}{*}{Theoretical} 
& \ref{subs:sha} & $\checkmark$ & $\checkmark$ & $\times$ & $\checkmark$ & \makecell{$\checkmark$ post-training\\ $\times$ during training} \\ 
\cline{2-7}
& \ref{subs:free-lunch} & $\times$ & $\checkmark$ & $\checkmark$ & \makecell{Single layer \\ only} & $\times$ \\ 
\cline{2-7}
& \ref{subs:mha-MHA} & $\times$ & \makecell{Around attention \\ only} & $\checkmark$ & $\checkmark$ & $\times$ \\ 
\cline{2-7}
& \ref{subs:weightshared} & $\times$ & $\checkmark$ & $\checkmark$ & \makecell{Shared weights \\ only} & $\times$ \\ 
\cline{2-7}
& \ref{subsection:Layernorm_n_basischange} & $\checkmark$ & Indep. & Indep. & Indep. & $\times$ \\ 
\hline
\multirow{2}{*}{Experimental} & \ref{app:mlp} & Indep. & $\checkmark$ & Indep. & Indep. & $\times$ \\ 
\cline{2-7}
& \ref{subs:full_gpt_train} & $\checkmark$ & $\checkmark$ & $\checkmark$ & $\checkmark$ & $\checkmark$ \\ 
\hline
\end{tabular}
}
\caption{Coverage of transformer components in theoretical analysis versus experimental validation. $\checkmark$ = feature covered, $\times$ = feature simplified/ignored, Indep.\ = result independent of feature.}

\end{table}

\section{Related Work}
\label{sec:related}

\paragraph{Architectural simplification.} Graef~\cite{graef2024transformer} formally proves that in skipless transformers without normalization, both $W_Q$ and $W_O$ can be eliminated simultaneously and leaves the question open whether elimination extends to practical architectures. We go further in several directions: (i) we analyze transformers \emph{with} skip connections by retaining $W_O$ to absorb basis changes, (ii) we characterize exactly when ReLU MLPs can absorb skip connections (Section~\ref{subsec:relu-skip}), (iii) we derive the functional form required for basis transformations through LayerNorm (Section~\ref{subsection:Layernorm_n_basischange}), and (iv) we validate through GPT-style pretraining. He \& Hofmann~\cite{ICLR2024_simplyfing_transformer} empirically study simplified parallel attention-MLP blocks; we focus on weight elimination within the original architecturGrouped-Query Attentione, rather than through block restructuring. Recent work showing that LayerNorm can be removed from pretrained models~\cite{heimersheim2024remove,baroni2025transformers} supports our no-normalization theoretical setting. Notably, Ji et al.~\cite{ji2025always} show that self-attention catastrophically fails to train without skip connections around attention, validating that our Theorem~\ref{thm:AttentionSkipOnlyQWElim} addresses the practically essential case.

\paragraph{Efficient attention.} Grouped-Query Attention~\cite{ALTdJ+23} and Multi-Query Attention~\cite{Shazeer2019FastTD} reduce parameters by sharing Key-Value projections. Linear attention methods~\cite{choromanski2020rethinking,wang2020linformer} reduce complexity. FlashAttention~\cite{dao2022flashattention} optimizes memory access. Our approach is orthogonal to these optimizations: Query weight elimination applies to standard, GQA, and MQA architectures alike, enabling multiplicative savings.

\paragraph{Weight sharing, tying, and recursive architectures.} In order to save memory Press and Wolf~\cite{press-wolf-2017-using} tie embedding and Language Modeling (LM) head weights with minimal performance degradation. This technique is further adopted in GPT-2~\cite{Radford2019LanguageMA} and many subsequent decoder-only models. Our theoretical work addresses both the tied and untied regimes.

Bai et al.~\cite{bai2019deep} study weight-shared transformers as implicit fixed-point solvers. ALBERT~\cite{lan2020albert} shares all parameters across layers, achieving 18$\times$ memory reduction at only modest performance cost. In this case the conditions of our theory are further relaxed: we systematically treat the case where all the skip connections are present, and only rely on the approximation of the change of basis through the LayerNorm. This implies that our theoretical investigation is highly relevant for this family of models. Finally, compared to weight sharing, our methods not only reduce memory, but also the computational footprint.

  Recursive models such as Tiny Recursion Models (TRM)~\cite{jolicoeurmartineau2025trm}, which apply a 2-layer block repeatedly, are natural candidates for Query weight elimination under our weight-sharing theorem. TRM uses RMSNorm, which our analysis covers (see Section~\ref{subsection:Layernorm_n_basischange}). Interestingly, the TRM authors found that embedding-head tying degraded performance, so TRM uses untied weights, satisfying our untying requirement and making Query weight elimination directly applicable.

\paragraph{Theoretical investigations.} Theoretical analysis of transformers typically trades architectural fidelity for mathematical tractability. We do the same to obtain tractable proofs, then verify our results empirically in full architectural complexity. For example, Yun et al.~\cite{yun2019transformersuniversal} establish universal approximation; Ildiz et al.~\cite{IHL+24} connect self-attention to Markov dynamics; Geshkovski et al.~\cite{GeshkEmergence,geshkovski2024mathematical} provide mean-field analyses of attention. Furthermore, our results bring the transformer variants analyzed in Theorem 3.1 and Appendix C.5 of~\cite{GeshkEmergence} closer to the ones used in practice.

\paragraph{Parameter-efficient fine-tuning and compression.} LoRA~\cite{lora} and related methods~\cite{Houlsby2019ParameterEfficientTL,qlora} achieve efficiency through low-rank adaptation during fine-tuning. Post-training compression via pruning~\cite{ma2023llm} and quantization~\cite{xu2023parameter,MWM+24} are also orthogonal to our approach. Query weight elimination applies at the architectural level, benefiting both pre-training and inference, and can be combined with these techniques for compound efficiency gains.

\section{Preliminaries}
\label{sec:prelim}

We now establish the notation used in the paper, as well as provide an observation on the self-attention that motivates the upcoming results.
\subsection{Notation}

We write $\text{Mat}(p,q)$ for the set of real matrices with p rows and q columns. We also denote by $GL(d)$ the set of invertible $d \times d$ matrices. Our results rely on $W_Q \in GL(d)$. This holds for almost every square matrix (in the Lebesgue sense)~\cite{Vershynin_2018} making it a reasonable assumption. Furthermore, near-singular $W_Q$ implies that certain directions contribute minimally to queries; replacing $W_Q$ with identity restores these directions.

Let $n$ be the sequence length, $d = d_\text{model}$ the embedding dimension, $h$ the number of heads, which brings us to the head dimension given by the integer $d_k = d/h$.

For each head $i \in \{1, \dots, h\}$, we define the head-specific weight matrices $W_Q^i, W_K^i, W_V^i \in  \text{Mat}(d,d_k)$.
The full layer parameters $W_Q, W_K, W_V \in  \text{Mat}(d,d)$ are constructed by concatenating these head blocks:
\begin{align*}
W_Q &= (W_Q^1 \mid \dots \mid W_Q^h) = \text{Concat}(W_Q^1, \dots, W_Q^h), \\
W_K &= (W_K^1 \mid \dots \mid W_K^h) = \text{Concat}(W_K^1, \dots, W_K^h), \\
W_V &= (W_V^1 \mid \dots \mid W_V^h) = \text{Concat}(W_V^1, \dots, W_V^h).
\end{align*}
Additionally, we denote by $W_O \in  \text{Mat}(d,d)$ the output projection matrix.

We let $M \in \{0, -\infty\}^{n \times n}$ denote the mask matrix defining the graph of allowed attention connections. The entry $M_{ij}$ determines whether position $i$ can attend to position $j$:
$$ M_{ij} = \begin{cases} 0 & \text{if token } i \text{ can attend to token } j \\ -\infty & \text{otherwise} \end{cases} $$
In the Bidirectional (Encoder-only) variant we have: $M = \mathbf{0}_{n \times n}$ and in the
Causal (Decoder-only) one we have: $M$ is strictly lower-triangular where $M_{ij} = -\infty$ if $j > i$ and $0$ otherwise.

The output of a single attention head $i$ is:
$$ \text{head}_i(X, W_Q, W_K, W_V) = \text{softmax}\left( \frac{(X W_Q^i)(X W_K^i)^t}{\sqrt{d_k}} + M \right) (X W_V^i). $$ where the softmax is being applied row-wise.

As an intermediary step, we define the \textit{Self-attention product} $\mathcal{S}_c$ as the concatenation of the heads:
$$ \mathcal{S}_c(X, W_Q, W_K, W_V) = \text{Concat}\bigl(\text{head}_1, \dots, \text{head}_h\bigr) \in  \text{Mat}(n,d). $$

The final \textit{Multi-Head Self-Attention} is the projection by $W_O$:
$$ \text{MHA}(X, W_Q, W_K, W_V, W_O) = \mathcal{S}_c(X, W_Q, W_K, W_V) \cdot W_O. $$ We consider $W_Q, W_K, W_V, W_O \in  \text{Mat}(d,d)$ as the attention parameters, and reducing one of them to $Id$ removes the claimed $25\%$ of the attention parameters. 

We use the definition of the Layer Normalization (LN) operator for an input vector $x \in \mathbb{R}^d$ as:
\begin{equation}
L_{\varepsilon, \gamma}(x) = \text{Ln}(x) = \gamma\odot \frac{x - \mu(x)}{\sqrt{\sigma^2(x) + \epsilon}},
\end{equation}
where $\mu(x) = \frac{1}{d}\sum_{i=1}^d x_i$ and $\sigma^2(x) = \frac{1}{d}\sum_{i=1}^d (x_i - \mu(x))^2$ denote the mean and variance of the components of $x$, respectively. Here, $\gamma \in \mathbb{R}^d$ represents a learnable rescaling parameter, $\gamma \odot x = \operatorname{diag}(\gamma)\, x$ denotes the componentwise (Hadamard) product, and $\epsilon > 0$ is a small constant. Note that we do not utilize a learnable bias (shift) parameter in this formulation. We often employ the abbreviated notation $\text{Ln}(x)$, where the dependence on the rescaling parameter $\gamma$ and the constant $\varepsilon$ is kept implicit. For $X \in  \text{Mat}(n,d)$ we denote by $\text{Ln}(X) \in  \text{Mat}(n,d)$ the application of the Layer Normalization on each row independently with the same fixed $\varepsilon, \gamma$.


\subsection{The Reparametrization Lemma}\label{sec:reparam}
We will start with a rather simple result on which we will build our theoretical investigation. 
The attention mechanism $\mathcal{S}_c$ depends on the input matrix $X$ \emph{only through} the projections $XW_Q$, $XW_K$, and $XW_V$, or, more formally:
\begin{observation}\label{obs:factorization}
\label{obs:reparam_invariance}
There exists a function $g$ such that:
$$\mathcal{S}_c(X, W_Q, W_K, W_V) = g(XW_Q, XW_K, XW_V)$$
\end{observation}

The proof of the preceding observation, while elementary, is deferred to Appendix \ref{subsection:block-notation}. The block notation and operators introduced there allow for a much cleaner exposition, but may cause unnecessary confusion if left in the main text.

The observation allows for the absorption of Query weights via a change of basis, which we formalize below.

\begin{lemma}[Reparametrization Lemma]\label{lem:repearmetrization}
Let $n, d \in \mathbb{N}$ and let $\Omega$ be any set. Consider a function 
$$f: \text{Mat}(n, d) \times \text{Mat}(d, d)^3 \to \Omega$$
that satisfies $f(X, W_Q, W_K, W_V) = g(XW_Q, XW_K, XW_V)$ for some function $g: \text{Mat}(n, d)^3 \to \Omega$.

Then $f$ is invariant under $(X, W_Q, W_K, W_V) \mapsto (X\Theta, \Theta^{-1}W_Q, \Theta^{-1}W_K, \Theta^{-1}W_V)$ for any $\Theta \in GL(d)$. In particular, if $W_Q \in GL(d)$, then there exist $\widetilde{W}_K, \widetilde{W}_V \in \text{Mat}(d,d)$ such that
$$f(X, W_Q, W_K, W_V) = f(XW_Q, I_d, \widetilde{W}_K, \widetilde{W}_V) \quad \text{for all } X \in \text{Mat}(n,d).$$
Explicitly, $\widetilde{W}_K = W_Q^{-1}W_K$ and $\widetilde{W}_V = W_Q^{-1}W_V$.
\end{lemma}

\begin{proof}
For any invertible $\Theta$, we have:
\begin{align*}
f(X\Theta, \Theta^{-1}W_Q, \Theta^{-1}W_K, \Theta^{-1}W_V) 
&= g(X\Theta \cdot \Theta^{-1}W_Q, \, X\Theta \cdot \Theta^{-1}W_K, \, X\Theta \cdot \Theta^{-1}W_V) \\
&= g(XW_Q, \, XW_K, \, XW_V) \\
&= f(X, W_Q, W_K, W_V).
\end{align*}
Setting $\Theta = W_Q$ gives $\Theta^{-1}W_Q = I_d$, with $\widetilde{W}_K = W_Q^{-1}W_K$ and $\widetilde{W}_V = W_Q^{-1}W_V$.
\end{proof}
Note that the choice of $W_Q$ above is arbitrary, and the same result under the analogous assumptions can be stated for $\Theta= W_K$ or $\Theta = W_V$.

We are now ready to dive deeper into the parameter reduction, and how it propagates between layers under different conditions.
\section{Theoretical analysis of the parameter reduction}\label{sec:theory}

In the previous notation, we have that the block can be written as:
\begin{align}
B(X) &= X+\text{MHA}\circ \text{Ln}_1(X) + \text{MLP} \circ \text{Ln}_2(X+\text{MHA}\circ \text{Ln}_1(X)) \label{eq:transformer_block_expanded} \\
&= (\text{Id} + \text{MLP}\circ \text{Ln}_2) \circ (\text{Id}+\text{MHA}\circ \text{Ln}_1)(X) \label{eq:transformer_block_factored}
\end{align}
where $\text{Ln}_i$ are two Layer Normalization layers with independent learnable vectors $\gamma_i$.
Let us now provide the theoretical investigation of the proposed improvement under different assumptions. The results are ordered by the subjective difficulty of deriving them. We remind the reader that such simplifications of the model have proven fruitful in the theoretical analysis of the transformer architecture \cite{GeshkEmergence, IHL+24, yun2019transformersuniversal, lan2020albert}.

\subsection{On some degrees of freedom in the Single Head and Multi-Head Attention}\label{subs:sha}
\subsubsection{The Single Head case}
We begin with single-head attention, the conceptually simplest case. Remarkably, this setting admits a clean result that holds regardless of normalization or skip connections: attention scores depend on queries and keys only through their product $W_QW_K^t$, and the output depends on values and output projection only through $W_VW_O$. Unlike the multi-head case analyzed in subsequent sections, the single-head reparametrization is local to the attention mechanism and does not require propagating basis transformations through other architectural components.

In the single head case, the attention mechanism reads:
$$\text{MHA}(X, W_Q, W_K, W_V, W_O) = \text{softmax}\left(\frac{XW_QW_K^tX^t}{c} + M \right)XW_V W_O$$

where $X \in \text{Mat}(n, d)$ is the input sequence, and $W_Q, W_K, W_V, W_O \in \text{Mat}(d \times d)$.

The following proposition confirms this intuition and serves as a warm-up for the multi-head and multi-layer cases analyzed in subsequent sections.

\begin{proposition}[Single-head redundancy]\leavevmode\\
For single-head attention, the four weight matrices $(W_Q, W_K, W_V, W_O)$ can be reduced to two, with the reduced form:
\begin{align*}
\mathrm{MHA}(X, W_Q, W_K, W_V, W_O)
 = \mathrm{MHA}(X, I_d, \widetilde{W}_K, \widetilde{W}_V, I_d)
\end{align*}
where $\widetilde{W}_K = W_K W_Q^t$ and $\widetilde{W}_V = W_V W_O$.
\end{proposition} 

\begin{proof} The attention computation depends only on the product $W_QW_K^t$ (for attention scores) and $W_VW_O$ (for the output). Setting $W_{QK} = W_KW_Q^t$ and $W_{VO} = W_VW_O$:
\begin{align*}
\text{softmax}\left(\frac{XW_QW_K^tX^t}{c}\right)XW_V W_O = \text{softmax}\left(\frac{X \cdot I_d \cdot W_{QK}^t \cdot X^t}{c}\right)X \cdot W_{VO} \cdot I_d
\end{align*}
This preserves the exact same output for any input sequence $X$.
\end{proof}

Weight decay during training may prevent the optimizer from naturally discovering this reduced form, as it penalizes $\|W_Q\|^2 + \|W_K\|^2 + \|W_V\|^2 + \|W_O\|^2$ rather than $\|W_QW_K^t\|^2 + \|W_VW_O\|^2$. That said, an already-trained such model can have its weights merged post-training.

\subsubsection{Extension to the Multi Head case}\label{subsection:multihead-blockdiag-invariance}
The previous result extends naturally in the multihead setting but in a weaker variant:

\begin{proposition}[Block-diagonal invariance]\label{prop:block-diag}
    Let $D = \mathrm{diag}(D_1, \ldots, D_h) \in GL(d)$ be any block-diagonal invertible matrix with blocks $D_i \in GL(d_k)$. Then:
    \[
    \mathcal{S}_c(X, W_Q, W_K, W_V) = \mathcal{S}_c\left(X, W_QD, W_K\left(D^\top\right)^{-1}, W_V\right).
    \]
\end{proposition}

The result formally identifies under-utilized degrees of freedom in the $(W_Q, W_K)$ pair and there probably is a similar amount of degrees of freedom in the case of $W_V$ when paired with $W_O$, but we leave such explorations to future work. In practice, weight decay partly addresses these degrees of freedom by biasing the optimizer toward a particular representative within each equivalence class.

That said, the degrees of freedom exposed by the result above are far from sufficient to guarantee the simplification of $W_Q$ or $W_K$. 

The simplification itself requires a more careful analysis of how the basis propagation happens in order to apply the Reparametrization Lemma~\ref{lem:repearmetrization}. Under a set of assumptions, it is what we will tackle in the remainder of the theoretical analysis section below.

\subsection{Single-Layer Query Weight Elimination (Free Lunch)}\label{subs:free-lunch}

Before addressing multi-layer elimination, we establish a foundational result: in any transformer without normalization, a \emph{single} layer's Query weight can always be eliminated, regardless of skip connection configuration. This ``free lunch'' requires no architectural modifications, requiring only weight reparametrization.

\begin{theorem}[Single-Layer Query Weight Elimination]\label{thm:free-lunch}
Consider an $L$-layer encoder-only or decoder-only transformer without normalization, with standard skip connections around both attention and MLP in each block, and with untied weights. Let $E, E_P, W_{\text{head}}$ be embedding, positional embedding, and output head weights, and let 
$$
\left(W^i_{MLP_u}, W^i_{MLP_d}, W^i_Q, W^i_K, W^i_V, W^i_O\right)_{i=1}^L
$$ 
be the layer weights.

For any $j \in \{1, \ldots, L\}$ with $W_Q^j \in GL(d)$, there exist modified weights 
$$
\left(\widetilde{E}, \widetilde{E}_P, \widetilde{W}_{\text{head}}, \left(\widetilde{W}^i_{MLP_u}, \widetilde{W}^i_{MLP_d}, \widetilde{W}^i_Q, \widetilde{W}^i_K, \widetilde{W}^i_V, \widetilde{W}^i_O\right)_{i=1}^L\right)
$$
such that $\widetilde{W}_Q^j = I_d$ and the transformer with these modified weights produces identical outputs to the original.
\end{theorem}

The proof appears in Appendix~\ref{app:proofs}.

\begin{remark}[Why not all layers?]\label{rem:why-not-all}
To eliminate all Query weights simultaneously, we would need layer $i$ to use basis $\Theta_i = W_Q^i$. This requires blocks to satisfy:
\[
\widetilde{B}_i(X\Theta_i) = B_i(X)\Theta_{i+1}
\]
with potentially $\Theta_i \neq \Theta_{i+1}$. However, the MLP skip connection forces intertwining (same basis in and out):
\[
\widetilde{\text{MLP}}_i(Z\Theta) + Z\Theta = (\text{MLP}_i(Z) + Z)\Theta
\]
This constraint admits two solutions: (i) remove the MLP skip, allowing $W_{down}^i$ to transform between bases (Theorem~\ref{thm:AttentionSkipOnlyQWElim}), or (ii) share weights so $\Theta_i = \Theta$ for all $i$ (Theorem~\ref{cor:weight-sharing}).
\end{remark}

\begin{remark}[Post-training application]\label{rem:post-training}
Heimersheim~\cite{heimersheim2024remove} and Baroni et al.~\cite{baroni2025transformers} show that LayerNorm can be removed from pretrained transformers (up to GPT-2 XL) with minimal degradation. For such models, our theorem applies directly post-training: a single layer's $W_Q$ can be eliminated through weight reparametrization alone, saving 25\% of that layer's attention parameters.
\end{remark}

\subsection{Multihead Attention with Skip Connection around the Attention \& without Normalization}\label{subs:mha-MHA}

Having established that single-layer elimination is always possible, we now address the challenge of eliminating Query weights from \emph{all} layers simultaneously. As Remark~\ref{rem:why-not-all} explains, this requires either removing the MLP skip connection or sharing weights across layers.

We first consider the attention-skip-only architecture: skip connections around attention but not around MLPs. This setting introduces a fundamental mechanism: basis transformations can propagate through the network via the \emph{intertwining} relation, where each block transforms between different input and output bases.

This architectural choice is partially supported by complementary evidence across domains. In vision transformers, removing attention skip connections causes catastrophic training failure, while removing MLP skip connections yields only modest degradation \cite{li2025always}. For language models, \cite{ICLR2024_simplyfing_transformer} show that parallel attention-MLP architectures sharing a single skip connection match standard sequential designs, suggesting architectural flexibility though not establishing clear hierarchy of importance.

We temporarily set aside layer normalization. This simplification is supported by evidence that normalization can be removed at inference with minimal degradation \cite{baroni2025transformers,heimersheim2024remove}, and that deep transformers can be trained without normalization when using signal-propagation-preserving attention modifications \cite{he2023deep}. As shown in Section~\ref{subsection:Layernorm_n_basischange}, incorporating normalization introduces technical complications obscuring the core mechanism. We therefore establish our result in this clearer setting first, then investigate conditions for extension through normalization layers.

\paragraph{Optimization versus trainability.} Skip connections confer two distinct benefits that are often conflated. \emph{Optimization} refers to training efficiency: skip connections smooth the loss landscape~\cite{li2018visualizing}, eliminate singularities that slow convergence~\cite{orhan2017skip}, and enable gradient flow through very deep networks~\cite{he2016deep}. Networks \emph{can} train without skip connections, just more slowly and with worse final performance. He et al.~\cite{he2023deep} show that ``vanilla'' transformers (without skip connections or normalization) can be trained using special modifications to self-attention, but require approximately $5\times$ more iterations to reach equivalent performance.

\emph{Trainability}, by contrast, refers to whether training succeeds at all. For self-attention specifically, Ji et al.~\cite{ji2025always} show that training \emph{catastrophically fails} without skip connections around attention, while removing them in other components (MLPs, convolutions) is much less consequential. Our Theorem~\ref{thm:AttentionSkipOnlyQWElim} addresses this practically essential case: transformers with skip connections around attention.

With these simplifications, we state our main result:

\begin{theorem}[Attention-Skip-Only Query Weight Elimination]\label{thm:AttentionSkipOnlyQWElim}
Consider an $L$-layer encoder-only or decoder-only transformer without normalization and with skip connections only around attention. Let $E, E_P, W_{\text{head}}$ be embedding, positional embedding, and output head weights (LM head for decoder-only, task head for encoder-only), and let 
$$
\left(W^i_{MLP_u}, W^i_{MLP_d}, W^i_Q, W^i_K, W^i_V, W^i_O\right)_{i=1}^L
$$ 
be the layer weights.

Then there exist modified weights 
$$
\left(\widetilde{E}, \widetilde{E}_P, \widetilde{W}_{\text{head}}, \left(\widetilde{W}^i_{MLP_u}, \widetilde{W}^i_{MLP_d}, I_d, \widetilde{W}^i_K, \widetilde{W}^i_V, \widetilde{W}^i_O\right)_{i=1}^L\right)
$$
such that the transformer with these weights produces identical outputs to the original.
\end{theorem}

\begin{remark}
    We can extend the result to the model without normalization but \textbf{with all the skip connections} (around the Attention and around the MLP) but then we can only reliably fix a single layer's $W_Q$ matrix with $Id$ and train the rest of the Query weight matrices.
\end{remark}

\subsection{Weight-Shared Multi-Layer Transformers with skip connections and no normalization}\label{subs:weightshared}

Weight sharing across transformer layers appears in both practical and theoretical contexts. Practically, ALBERT \cite{lan2020albert} employs full parameter sharing to achieve up to 18$\times$ parameter reduction compared to BERT while maintaining competitive performance, demonstrating that layer-specific parameters may be inessential for many tasks. Theoretically, weight sharing enables multiple analytical frameworks: deep equilibrium models \cite{bai2019deep} study transformers as fixed-point solvers of a single repeated block, while mean-field approaches \cite{geshkovski2024mathematical} analyze the continuous-time limit as $L \to \infty$, revealing clustering phenomena and long-time behavior of token representations in such a case.

Theorem~\ref{cor:weight-sharing} provides structural evidence that, in weight-shared architectures, the query matrix can be eliminated without loss of generality. While established here in a simplified setting, this canonical form offers analytical indications for simplifying fixed-point and continuous-time limits analysis.

Having established query elimination for a single layer, we now consider the weight-shared regime where all layers use identical parameters.
\begin{theorem}[Weight-Shared Query Weight Elimination]
\label{cor:weight-sharing}
Consider an $L$-layer encoder-only or decoder-only transformer without normalization where all layers share the same block parameters. Let $E, E_P, W_{\text{head}}$ be embedding, positional embedding, and output head weights, and let 
\[
\left(W_{MLP_u}, W_{MLP_d}, W_Q, W_K, W_V, W_O\right)
\]
be the shared layer weights with $W_Q \in GL(d)$.

Then there exist modified weights 
\[
\left(\widetilde{E}, \widetilde{E}_P, \widetilde{W}_{\text{head}}, \widetilde{W}_{MLP_u}, \widetilde{W}_{MLP_d}, I_d, \widetilde{W}_K, \widetilde{W}_V, \widetilde{W}_O\right)
\]
such that the transformer with these weights produces identical outputs to the original.
\end{theorem}

\begin{remark}
We note that in the case of tied LM head and embedding weights ($W_{LM} = E^t$), the single-layer reduced model must untie these weights to maintain functional equivalence. Specifically, the reduced model requires $\widetilde{W}_{LM} = W_Q^{-1}E^t$ and $\widetilde{E} = EW_Q$, which satisfy $\widetilde{W}_{LM} \neq \widetilde{E}^t$ unless $W_Q$ is orthogonal.
\end{remark}

We validate our theoretical results by training decoder-only language models from scratch with $W_Q = I_d$ and comparing them to standard baselines. Our experiments confirm that Query weights are indeed redundant: models trained without them achieve comparable performance to standard models while using fewer parameters.
\section{Experiments}\label{sec:experiments}

We validate our theoretical results by training GPT-style models from scratch with $W_Q = I_d$. Preliminary experiments on MLP basis transfer approximation are in Appendix~\ref{app:mlp}.

\subsection{Pretraining of GPT-Style Models}\label{subs:full_gpt_train}

\textbf{Architecture.} We use Karpathy's NanoGPT~\cite{Karpathy} implementation of GPT-2/GPT-3-small~\cite{Radford2019LanguageMA,brown2020language}: 12 layers, 12 heads, $d = 768$, MLP hidden dimension $4d = 3072$, GELU activations~\cite{hendrycks2016gaussian}, LayerNorm~\cite{ba2016layer}, sequence (context) length 1024, GPT-2 BPE tokenizer. Compared to GPT-2 and 3, but consistent with modern frontier models and Karpathy's implementation, we omit the bias parameters.

\textbf{Training \& Evaluation.} All models are trained for 60k gradient steps on OpenWebText~\cite{Gokaslan2019OpenWeb} with mixed-precision training, using AdamW~\cite{loshchilov2018decoupled} with $\beta_1 = 0.9$, $\beta_2 = 0.95$, gradient clipping at 1.0, weight decay $0.1$, and $\sim$490k tokens per gradient step ($\sim$29.5B tokens total, approximately $12\times$ Chinchilla-optimal~\cite{hoffmann2022training} for models of this size). The cosine learning rate schedule completes fully at 60k steps with 2k steps of warmup, ensuring all models are evaluated at convergence rather than mid-schedule. Baseline models use the standard GPT-2/GPT-3 learning rate of $6 \times 10^{-4}$ decaying to $6 \times 10^{-5}$; reduced models use higher peak learning rates (see Table~\ref{tab:results} and Section~\ref{sec:discussion}). To ensure a fair comparison, we pre-generate all training and evaluation batch indices from a fixed random seed and reuse them across all model variants, so every configuration sees identical data in identical order. Validation loss is estimated every 1000 steps by averaging over 2400 sequences ($\approx$2.5M tokens) drawn independently from the validation split. Training uses a single NVIDIA RTX 5090 GPU with FlashAttention~\cite{dao2022flashattention}. All configurations use tied embedding/LM-head weights.

\textbf{Model variants.} To isolate the effect of Query weight elimination from parameter count differences, we compare five configurations (Table~\ref{tab:results}). Following scaling law conventions~\cite{hoffmann2022training}, Table~\ref{tab:results} and Figure~\ref{fig:training} report non-embedding parameters. Throughout, \emph{reduced} refers to architectures with $W_Q = I_d$:
\begin{itemize}
    \item \textbf{Baseline (124M):} Standard GPT-2 architecture with full $(W_Q, W_K, W_V, W_O)$ and MLP hidden dimension $4 \times d$ (84.95M non-embedding parameters).
    \item \textbf{Baseline (117M, Smaller MLP):} Standard architecture with MLP hidden dimension $3.5 \times d$ (77.88M non-embedding parameters).
    \item \textbf{Baseline (118M, Smaller $d$):} Standard architecture with $d = 744$ instead of 768 (79.73M non-embedding).
    \item \textbf{Reduced (117M):} Reduced architecture with $W_Q = I_d$ and the same $4 \times d$ MLP. Eliminating $W_Q$ reduces non-embedding parameters to 77.88M.
    \item \textbf{Reduced (124M, Larger MLP):} Reduced architecture with saved parameters reallocated to MLP ($4.5 \times d$ hidden dimension, 84.95M non-embedding).
\end{itemize}

\textbf{Practical adjustments.} Two modifications are necessary for the reduced architecture:
\begin{enumerate}
    \item \textbf{Attention Scaling Correction:} We adopt a scaling factor of $\frac{1}{2\sqrt{d_k}}$ instead of the standard $\frac{1}{\sqrt{d_k}}$. The intuition: with $W_Q = I_d$, queries are coordinate slices of the input rather than learned projections. At initialization, this yields attention scores with approximately $1.8\times$ larger standard deviation than the baseline (see Appendix~\ref{app:scaling} for the full derivation). The factor of $\frac{1}{2}$ compensates for this to prevent early softmax saturation.
    
\item \textbf{Learning Rate:} We find that the reduced architecture benefits from higher peak learning rates than the standard GPT-2/GPT-3 value of $6 \times 10^{-4}$: we use $1.6 \times 10^{-3}$ for the base reduced model and $2.2 \times 10^{-3}$ for the larger-MLP variant, both decaying to $2 \times 10^{-5}$. Standard weight decay ($0.1$) is used throughout. In the standard architecture, attention logits depend on the product $W_Q W_K^t$, quadratic in the learned weight matrices; in the reduced architecture, the logits are linear in $W_K$ alone, which admits higher learning rates. We tested weight decay values across $[0.025, 0.1]$: lower values tolerated only slightly smaller learning rates than the best configurations reported here, and higher weight decay within this range yielded slightly better final performance. We did not explore values above $0.1$, which may yield further improvements; this is left for future work. We note that in pre-LayerNorm transformers, weight decay is not applied to the learnable LayerNorm scaling $\gamma$, and therefore acts as an optimization stabilizer rather than a regularizer~\cite{NEURIPS_WDRole}.
    
\end{enumerate}

\begin{table}[t]
\centering
\caption{Model configurations and validation loss at 60k steps (completed cosine schedule, $\sim$29.5B tokens, $\sim$$12\times$ Chinchilla-optimal).}
\label{tab:results}
\small
\begin{tabular}{lccc|cc}
\toprule
 & Baseline & Baseline & $W_Q = I$ & Baseline & $W_Q = I$ \\
 & {\scriptsize(smaller MLP)} & {\scriptsize(smaller $d$)} & {\scriptsize(117M)} & {\scriptsize(124M)} & {\scriptsize(larger MLP)} \\
\midrule
Trainable $W_Q$ & $\checkmark$ & $\checkmark$ & $\times$ & $\checkmark$ & $\times$ \\
Total params & 117.30M & 117.92M & 117.30M & 124.37M & 124.37M \\
Non-emb params & 77.88M & 79.73M & 77.88M & 84.95M & 84.95M \\
MLP hid.\ mult. & $3.5\times$ & $4\times$ & $4\times$ & $4\times$ & $4.5\times$ \\
Attention scale & $1/\sqrt{64}$ & $1/\sqrt{62}$ & $1/(2\sqrt{64})$ & $1/\sqrt{64}$ & $1/(2\sqrt{64})$ \\
Weight decay & $0.1$ & $0.1$ & $0.1$ & $0.1$ & $0.1$ \\
Peak LR & $6\text{e-}4$ & $6\text{e-}4$ & $1.6\text{e-}3$ & $6\text{e-}4$ & $2.2\text{e-}3$ \\
Min LR & $6\text{e-}5$ & $6\text{e-}5$ & $2\text{e-}5$ & $6\text{e-}5$ & $2\text{e-}5$ \\
\midrule
Val loss (60k) & 2.955 & 2.956 & 2.944 & 2.944 & \textbf{2.929} \\
vs.\ Baseline & $+0.37\%$ & $+0.40\%$ & $+0.00\%$ & Ref. & $\mathbf{-0.52\%}$ \\
\bottomrule
\end{tabular}
\end{table}

\begin{figure}[t]
\centering
\includegraphics[width=0.7\textwidth]{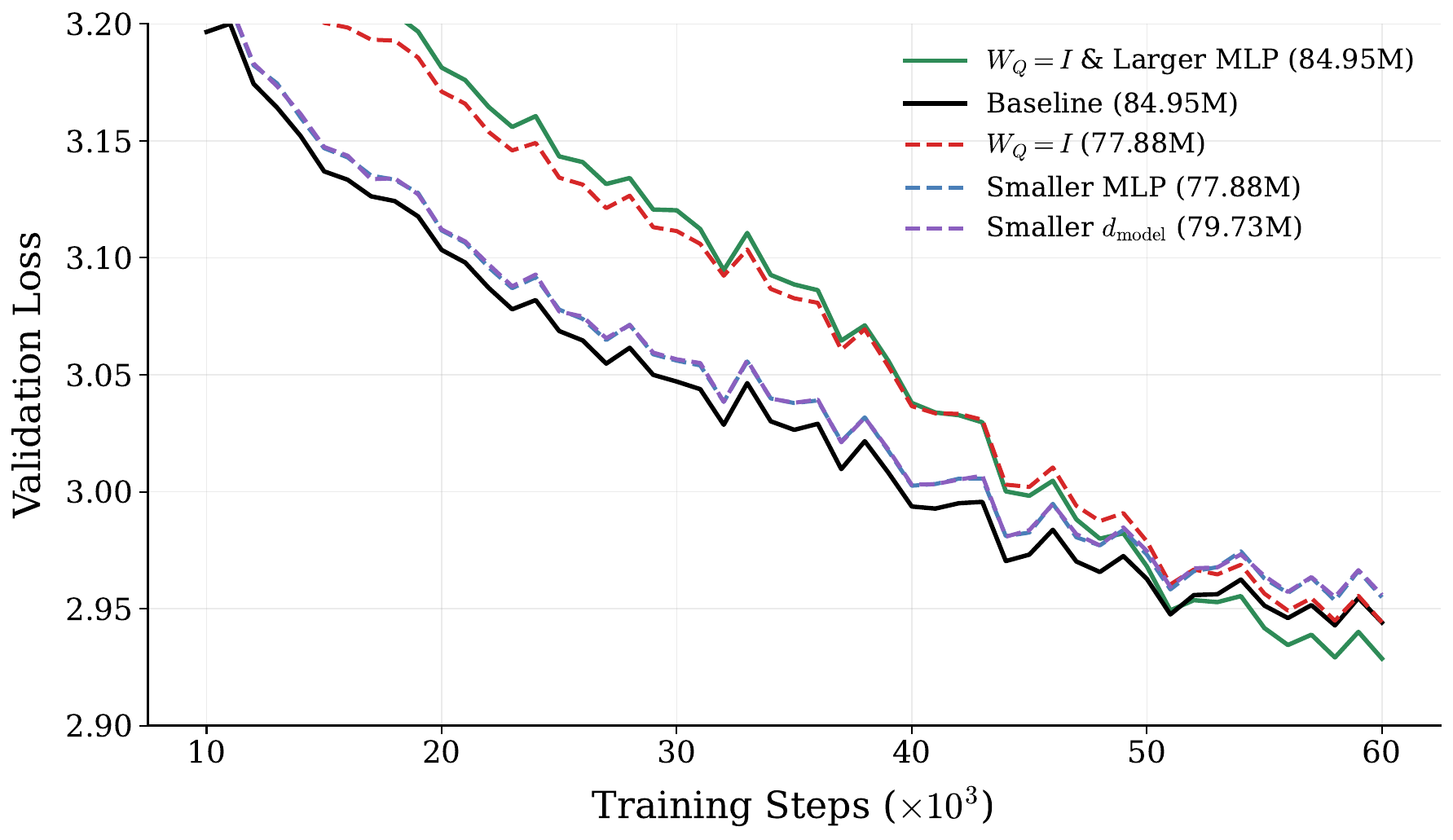}
\caption{Validation loss during training (non-embedding parameters in parentheses). All models complete their full cosine learning rate schedule at 60k steps.}
\label{fig:training}
\end{figure}

\textbf{Results.} Table~\ref{tab:results} and Figure~\ref{fig:training} summarize our findings:

\begin{enumerate}
    \item \textbf{Query weights are largely redundant.} The reduced model ($W_Q = I_d$, 77.88M non-embedding parameters) matches the full baseline (84.95M) in validation loss (2.944 vs.\ 2.944), despite having 8\% fewer non-embedding parameters. The 7M parameters in $W_Q$ contribute nothing to generalization.
    
    \item \textbf{Parameter reallocation improves performance.} When saved parameters are reallocated to the MLP ($4.5 \times d$ hidden dimension), the reduced architecture achieves the \emph{best} validation loss (2.929 vs.\ 2.944 for the baseline at equal parameter count), a clear improvement.
    
    \item \textbf{$W_Q$ parameters are specifically redundant.} At equal non-embedding parameter count ($\sim$78M), standard baselines with smaller MLP or smaller $d$ achieve 2.955--2.956, meaningfully worse than the full 84.95M baseline. Reducing parameters in the MLP or embedding dimension costs performance; removing $W_Q$ does not. This identifies the Query projection as the one location in the architecture where parameters are wasted.
    
    \item \textbf{Standard weight decay suffices.} All models use weight decay $0.1$, the standard value. No special regularization adjustment is needed for the reduced architecture.
\end{enumerate}

\section{Discussion}\label{sec:discussion}

\textbf{Learning rate schedule.} The reduced architecture benefits from peak learning rates $2.7$--$3.7\times$ higher than the standard GPT-2/GPT-3 value of $6 \times 10^{-4}$. This is consistent with the simplified attention landscape: with logits linear in $W_K$ rather than quadratic in $W_Q W_K^t$, the loss surface admits larger steps without destabilization. Interestingly, while the peak learning rate is much higher than the baseline, the minimum learning rate is substantially \emph{lower}: $2 \times 10^{-5}$ versus the baseline's $6 \times 10^{-5}$, giving a peak-to-minimum ratio of $80$--$110\times$ compared to the standard $10\times$. This deeper decay proved beneficial; using the conventional $\frac{1}{10}$ ratio (i.e., a higher minimum learning rate) yielded worse results. The combination of aggressive early learning with deep late-stage decay may reflect the reduced architecture's ability to make large initial updates without destabilization, followed by fine-grained convergence at very low learning rates.

\textbf{Limitations.} Our experiments validate the core result on fully converged training runs at the 117M-124M parameter scale, training approximately $12\times$ beyond Chinchilla-optimal token count. Multi-seed validation, larger scales, and downstream benchmarks represent natural next steps. The reduced architecture's optimal learning rate differs from the baseline, requiring a hyperparameter search over learning rate; however, weight decay and all other hyperparameters remain unchanged.

\textbf{Extensions and future directions.} The reduction mechanism relies only on a linear projection structure and extends to: Rotary Position Embeddings~\cite{su2021ROPEroformer} (fixed rotations satisfy the condition of the Reparametrization Lemma), Grouped-Query Attention~\cite{ALTdJ+23} (transformations act per query-head group), Mixture-of-Experts~\cite{jacobs1991adaptive,shazeer2017outrageously}, as well as QK-normalization~\cite{henry2020querykey}. The analysis extends without any modification to encoder-only architectures, with potentially larger relative savings. Technical directions include exploring untied embedding configurations.

While we have chosen the Query projection as the primary elimination target due to its seamless compatibility with the KV cache and GQA, insights from the parameter-efficient fine-tuning literature serve as a useful empirical guideline for weight importance. Specifically, studies often observe that the Key matrix has the smallest effect on performance, suggesting $W_K$ may theoretically be an even better candidate for structural simplification.

Furthermore, while our results demonstrate that \emph{linear} Query projections provide no benefit over identity, exploring \emph{nonlinear} variants presents a natural extension. Building on these same empirical guidelines, if Key modifications yield minimal effects, the Query conversely serves as a stronger driver of expressivity and a better target for nonlinear enhancement. A nonlinear Query transformation with a skip connection, i.e., $Q(X) = X + N(X; \theta)$ where $N$ is a nonlinearity with approximately the same number of training parameters as $W_Q$, may offer additional expressivity while preserving training stability, compute, and memory budgets. Enhancing the Query represents a practical sweet spot for autoregressive generation. Conversely, if KV caching or GQA are not required (as is the case in many encoder-only models), testing such nonlinearities on the Value projection ($W_V$) instead could also prove to be an excellent starting point. 

Additionally, recent work on Hyper-Connections~\cite{zhu2024hyperconnections} and Manifold Constrained Hyper-Connections~\cite{xie2025mhc} expands the residual stream width and diversifies connectivity patterns while maintaining stability through doubly stochastic constraints. Investigating Query, Key, or Value weight elimination in such architectures presents an interesting direction.

\section{Conclusion}\label{sec:conclusion}
We prove that the Query weight matrix $W_Q$ in multi-head attention can be eliminated through basis transformations under simplified architectural assumptions, reducing attention parameters by 25\% per layer. Empirical validation confirms that models with $W_Q = \text{Id}$ match baselines when appropriately tuned, and \emph{exceed} baselines when saved parameters are reallocated to the MLP, consistent with our theoretical analysis identifying MLP expressivity as a potential limiting factor (Sections~\ref{subsection:Layernorm_n_basischange} and \ref{subsec:relu-skip}). The results suggest architectural redundancy in the Query-Key-Value triplet and raise questions about which transformer components are necessary for expressivity versus artifacts of design history.
\section{Appendix}

\subsection{Block-notation for the multi-head
attention}\label{subsection:block-notation}

We proceed as follows:

First, we prove Observation~\ref{obs:factorization} using the standard 
notation for the convenience of the reader. 

Second, we introduce the Block Hadamard Products (as variants of those 
introduced in~\cite{Horn1991}) to model the multi-head mechanism. In this 
notation, the result becomes mathematically trivial: it is manifestly 
clear that the self-attention product $\mathcal{S}_c$ is a function of 
$XW_Q$, $XW_K$, and $XW_V$ alone, as the block-wise operations never 
access $X$ except through these projections. The structural redundancy 
of the query weights is thus rendered immediate.

The notation is motivated by and provides an index-free formalization of 
standard implementations (e.g.,~\cite{Karpathy, PyTorch, Attnisallyouneed}) 
which use $XW_Q$, $XW_K$, $XW_V$ (after a split from $XW_{QKV}$), and 
reshape them into per-head blocks $(XW_Q)^1, \dots, (XW_Q)^h$ (and 
similarly for $XW_K$, $XW_V$). Our notation formalizes the 
transition: from the full products $XW_Q$, $XW_K$, $XW_V$ to the 
(parallel) per-head processing.

 We remind the reader that we've fixed $n$ (sequence length),
$d = d_\text{model}$ (embedding dimension), $h$ (number of attention heads) and $d_k = d/h$
denotes the dimension per head.

\paragraph{Slice operator.}
Let $m, n$ be two integers such that $0 < m < n \le d$, and in what follows, we will select
them to be of the form $m = a \times d_k$, $n = (a+1) \times d_k$ for $0 < a < h$. We define
the slice operator as the canonical projection
$P_{[m,n]}: \mathbb{R}^{d} \ni x = (x_1, \dots, x_{d}) \mapsto P_{[m,n]}(x) = xP_{[m,n]}
= (x_m, x_{m+1}, \dots, x_n)$, and we remind the reader that we consider vectors as row
vectors here. 
For the head $i$ we will write $P_i = P_{[(i-1) \times d_k+1, i \times d_k]}$.

We remind the reader that in the notation we wrote
$W_Q = (W_Q^1 | \dots | W_Q^h) = \text{concat}(W_Q^1, \dots, W_Q^h)$, and we wrote it
similarly for $W_K, W_V$. Now notice that $W_Q P_i = W_Q^i$ and we again have the
analogous result for $W_K, W_V$.

Moreover, associativity gives:
\begin{lemma}[Head weights interchange]\label{lem:block-product}
For any $X \in  \text{Mat}(n,d)$ and $W \in  \text{Mat}(d,d)$, the $i$-th head
block of the product equals the product with the $i$-th head block:
$$(XW)^i \;=\; XW^i.$$
In particular, this applies to $W = W_Q$, $W_K$, and $W_V$, giving
$(XW_Q)^i = XW_Q^i$ and analogously for the Key and Value projections.
\end{lemma}
\begin{proof}
By definition of the block decomposition, $(XW)^i = (XW)P_i$. By the observation above,
$W^i = WP_i$. Associativity of matrix multiplication gives
$(XW)P_i = X(WP_i)$, hence $(XW)^i = XW^i$.
\end{proof}

We can complete the proof of the observation here by using the lemma to remark that \begin{align*}
\operatorname{head}_i(X, W_Q, W_K, W_V)
  &= \operatorname{softmax}\!\left(
       \frac{(XW_Q^i)(XW_K^i)^t}{\sqrt{d_k}} + M
     \right)(XW_V^i) \\
  &= \operatorname{softmax}\!\left(
       \frac{(XW_Q)^i\bigl((XW_K)^i\bigr)^t}{\sqrt{d_k}} + M
     \right)(XW_V)^i\\
  &= g_i(XW_Q, XW_K, XW_V)
\end{align*}
for an appropriate function $g_i$, and then \begin{align*}\mathcal{S}_c (X, W_Q, W_K, W_V) &= \operatorname{Concat}\left(g_1(XW_Q, XW_K, XW_V), \dots, g_h(XW_Q, XW_K, XW_V)\right) \\
&= g(XW_Q, XW_K, XW_V)\end{align*}
for the appropriate function $g$.

Instead, we make the deliberate choice to introduce a different notation in which this property becomes visible right from the definition of the multihead attention. The reader can skip the remainder of this subsection.

\paragraph{Block-wise operations.} We decompose matrices into head-wise blocks:
\begin{itemize}
\item For $F \in \text{Mat}(n,d)$, write $F = (F_1 | \cdots | F_h)$ where each
      $F_i \in \text{Mat}(n,d_k)$ (e.g.\ $F = XW_Q$ or $F = XW_V$).
\item For $G \in \text{Mat}(d, n)$, write
      $G = \begin{pmatrix} G_1 \\ \vdots \\ G_h \end{pmatrix}$ where each
      $G_i \in \text{Mat}(d_k, n)$ (e.g.\ $G = (XW_K)^t$).
\item For $A \in \text{Mat}(n, hn)$, write $A = (A_1 | \cdots | A_h)$ where each
      $A_i \in \text{Mat}(n, n)$ (e.g.\ the matrix of per-head attention logits or weights).
\end{itemize}
We recall that, in this notation, the standard matrix multiplication of $F$ and $G$ can be
written as $FG = \sum_{i=1}^h F_i G_i$.

To express per-head operations, we define two
\emph{Block Hadamard Products}:
$$F \squaredtimes^t G = (F_1 G_1 | \cdots | F_h G_h) \in  \text{Mat}(n,hn),$$
$$A \squaredtimes F = (A_1 F_1 | \cdots | A_h F_h) \in  \text{Mat}(n,d).$$
These are variants of Hadamard block products introduced in~\cite{Horn1991}.

The \emph{block softmax} applies softmax and optional masking independently to each head:
$$\text{Softmax}^{\squaredtimes}_M(A_1 | \cdots | A_h)
= \bigl(\text{softmax}(M + A_1) \,|\, \cdots \,|\, \text{softmax}(M + A_h)\bigr),$$
where $A_i \in  \text{Mat}(n, n)$ are attention logits and $M \in  \text{Mat}(n, n)$
is a mask, and we apply the softmax row-wise.

\paragraph{Multi-head attention in block notation.} Multi-head self-attention written in
the notation given above then reads:
\begin{align}\label{eq:sc-block}
\mathcal{S}_c(X, W_Q, W_K, W_V)
  &= \text{Softmax}^{\squaredtimes}_M\!\left(
       \frac{1}{\sqrt{d_k}}\,(XW_Q) \squaredtimes^t (XW_K)^t
     \right) \squaredtimes (XW_V) \\
\text{MHA}(X, W_Q, W_K, W_V, W_O)
  &= \mathcal{S}_c(X, W_Q, W_K, W_V) \cdot W_O.
\end{align}
The Observation~\ref{obs:factorization} now follows directly from the definition (\ref{eq:sc-block}).
All we need to do now is prove that this definition does indeed match the standard one.

\begin{proposition}[Equivalence with the standard definition]\label{prop:block-equiv}
The block-notation expression~\eqref{eq:sc-block} for $\mathcal{S}_c$ coincides with the
standard concatenation of heads:
$$\mathcal{S}_c(X, W_Q, W_K, W_V)
  = \operatorname{Concat}(\operatorname{head}_1, \dots, \operatorname{head}_h).$$
\end{proposition}

\begin{proof}
Denote by $\mathcal{S}_c^{\squaredtimes}$ the right-hand side
of~\eqref{eq:sc-block}. We show that its $i$-th block equals
$\operatorname{head}_i$ by tracking the $i$-th block through each
operation in sequence. Recall that
$$\operatorname{head}_i(X, W_Q, W_K, W_V)
  = \operatorname{softmax}\!\left(
      \frac{(XW_Q^i)(XW_K^i)^t}{\sqrt{d_k}} + M
    \right)(XW_V^i).$$

\medskip\noindent\emph{Step 1: Logits.}\;
By definition of $\squaredtimes^t$, the product
$(XW_Q) \squaredtimes^t (XW_K)^t$ concatenates the $h$ blocks
$(XW_Q)^i \bigl((XW_K)^i\bigr)^t$ for $i = 1, \dots, h$.
By Lemma~\ref{lem:block-product}, $(XW_Q)^i = XW_Q^i$ and
$(XW_K)^i = XW_K^i$, so the $i$-th block of the scaled logits is
$$\frac{1}{\sqrt{d_k}}\,(XW_Q^i)(XW_K^i)^t.$$

\medskip\noindent\emph{Step 2: Softmax.}\;
By definition, $\operatorname{Softmax}^{\squaredtimes}_M$ applies the
row-wise softmax (with mask $M$) independently to each $n \times n$
block and concatenates the results. Denoting the output by
$\mathcal{A} = (\mathcal{A}_1 | \cdots | \mathcal{A}_h)$, the $i$-th
block is
$$\mathcal{A}_i
  = \operatorname{softmax}\!\left(
      \frac{(XW_Q^i)(XW_K^i)^t}{\sqrt{d_k}} + M
    \right).$$

\medskip\noindent\emph{Step 3: Value weighting.}\;
By definition of $\squaredtimes$, the final product
$\mathcal{A} \squaredtimes (XW_V)$ multiplies the $i$-th
$n \times n$ block of $\mathcal{A}$ with the $i$-th
$n \times d_k$ block of $(XW_V)$ and concatenates the results:
$$\mathcal{A} \squaredtimes (XW_V)
  = \bigl(\mathcal{A}_1 (XW_V)^1 \,\big|\, \cdots \,\big|\,
          \mathcal{A}_h (XW_V)^h\bigr).$$
By Lemma~\ref{lem:block-product}, $(XW_V)^i = XW_V^i$, so the
$i$-th block of $\mathcal{S}_c^{\squaredtimes}$ is
$$\bigl[\mathcal{S}_c^{\squaredtimes}\bigr]_i
  = \mathcal{A}_i \cdot XW_V^i
  = \operatorname{softmax}\!\left(
      \frac{(XW_Q^i)(XW_K^i)^t}{\sqrt{d_k}} + M
    \right)(XW_V^i)
  = \operatorname{head}_i(X, W_Q, W_K, W_V).$$

\medskip\noindent\emph{Conclusion.}\;
Since $\mathcal{S}_c^{\squaredtimes}$ and
$\operatorname{Concat}(\operatorname{head}_1, \dots, \operatorname{head}_h)$
are both concatenations of the same $h$ blocks, they are equal.
\end{proof}

\begin{proof}[Proof of Observation~\ref{obs:factorization}]
By Proposition~\ref{prop:block-equiv}, the block-notation formula~\eqref{eq:sc-block} is an
equivalent expression for $\mathcal{S}_c$. Setting $A = XW_Q$, $B = XW_K$, $C = XW_V$,
it reads
$$\mathcal{S}_c(X, W_Q, W_K, W_V)
  = \operatorname{Softmax}^{\squaredtimes}_M\!\left(
      \tfrac{1}{\sqrt{d_k}}\, A \squaredtimes^t B^t
    \right) \squaredtimes C
  \;=:\; g(A, B, C).$$
The right-hand side depends on $X$ only through the products $A$, $B$, $C$.
\end{proof}

We finish this section by providing the rather simple proof of
Proposition~\ref{prop:block-diag}: block-diagonal transformations cancel head-by-head inside
$\squaredtimes^t$, leaving $\mathcal{S}_c$ unchanged. We note that this proof can be done without an issue in the standard notation as well.

\begin{proof}[Proof of Proposition~\ref{prop:block-diag}]
Since $D = \operatorname{diag}(D_1, \dots, D_h)$ is block-diagonal, the $i$-th head block of
$W_Q D$ is $W_Q^i D_i$ and the $i$-th head block of $W_K(D^\top)^{-1}$ is
$W_K^i (D_i^\top)^{-1}$. Therefore, the $i$-th block of the attention logits satisfies:
$$(XW_Q^i D_i)\bigl(XW_K^i (D_i^\top)^{-1}\bigr)^t
  = XW_Q^i \, D_i \, D_i^{-1} \, (XW_K^i)^t
  = XW_Q^i \, (XW_K^i)^t,$$
where we used $\bigl((D_i^\top)^{-1}\bigr)^t = D_i^{-1}$.
Since this holds for every head, $(XW_Q D) \squaredtimes^t (XW_K(D^\top)^{-1})^t
= (XW_Q) \squaredtimes^t (XW_K)^t$. The value projection $W_V$ is unchanged, so
$\mathcal{S}_c(X,\, W_Q D,\, W_K(D^\top)^{-1},\, W_V)
= \mathcal{S}_c(X,\, W_Q,\, W_K,\, W_V)$.
\end{proof}

\subsection{Attention Scaling Derivation}\label{app:scaling}

We derive the corrective scaling factor $\frac{1}{2\sqrt{d_k}}$ for the reduced architecture, even though we believe that this constant should ultimately be handled empirically. Let $x \in \mathbb{R}^{d}$ be the input to the attention layer. For the simplicity of the derivation, we may assume that the input rows have unit norm, $\|x\|^2 = 1$. Weights are initialized from $\mathcal{N}(0, \sigma^2)$ with $\sigma=0.02$.

\paragraph{Standard Baseline ($q_i = xW_Q^i$).} The query for head $i$ is a linear projection where $W_Q^i \in  \text{Mat}(d,d_k)$. Since each entry of $W_Q^i \sim \mathcal{N}(0, \sigma^2)$, the variance of each component of $q_i$ is $\|x\|^2 \sigma^2 = \sigma^2$. The attention score $S_{\text{std}} = q_i \cdot (xW_K^i)^\top$ is the dot product of two $d_k$-dimensional vectors with component variance $\sigma^2$, yielding an initial standard deviation of:
\[
\text{StdDev}(S_{\text{std}}) = \sigma^2\sqrt{d_k} = (0.02)^2\sqrt{64} = 0.0032.
\]

\paragraph{Reduced Architecture ($W_Q = I$).} The query for head $i$ is defined as a coordinate slice of the input: $q_i = x_{[i \cdot d_k : (i+1) \cdot d_k]}$. Because the total norm $\|x\|^2 = 1$ is distributed across the $h$ heads, the expected squared norm of this slice is $\|q_i\|^2 = \frac{1}{h}$. The attention score $S_{\text{red}} = q_i \cdot (xW_K^i)^\top$ has a variance that depends only on the key weight variance:
\[
\text{Var}(S_{\text{red}}) = \|q_i\|^2 \sigma^2 = \frac{\sigma^2}{h} \implies \text{StdDev}(S_{\text{red}}) = \frac{\sigma}{\sqrt{h}} = \frac{0.02}{\sqrt{12}} \approx 0.0058.
\]

\paragraph{Corrective Factor.} Comparing the two regimes, the initial scores in the reduced architecture are approximately $1.8\times$ larger than the baseline ($\frac{0.0058}{0.0032} \approx 1.8$). To prevent early softmax saturation and maintain a starting variance consistent with standard transformers, we introduce a corrective factor to the scaling. Since $1/1.8 \approx 0.55$, we adopt the factor $\frac{1}{2}$ as a clean and effective approximation, resulting in the modified scaling $\frac{1}{2\sqrt{d_k}}$.

\subsection{LayerNorm \& Change of Basis via the (skip +) MLP}\label{subsection:Layernorm_n_basischange}

In Sections~\ref{subsection:Layernorm_n_basischange}–\ref{subsec:relu-skip}, we adopt the standard mathematical convention: vectors are column vectors and linear maps act by left multiplication ($Ax$), in contrast to the ML convention ($xA$) used in the main text.

We recall that recent empirical studies \cite{heimersheim2024remove, baroni2025transformers} suggest that Layer Normalization can often be omitted at least at inference time with little to no degradation in performance. While such evidence points to its potential disposability, we also investigate from the complementary perspective of asking what structural constraints arise if LayerNorm is in fact kept as-is. In particular, we investigate how the presence of LayerNorm interacts with surrounding linear and nonlinear components, and what reparametrizations would be required to preserve the change of basis.

Concretely, we ask when there exist a modified multilayer perceptron $MLP'$ and a diagonal matrix $D'$ such that
\begin{equation}
    D' \, L_{\varepsilon}\!\big(x+MLP'( x)\big) \;\approx\; \Theta D \, L_{\varepsilon}\!\big(x+MLP(x)\big),
\end{equation}
where $L_{\varepsilon}$ denotes the $\varepsilon$-regularized LayerNorm operator, $D$ is the original learned diagonal scaling associated with the normalization, and $W_Q$ is the subsequent linear projection. In order to provide the sufficient condition, we will rely on the following Lemma:

\begin{lemma}[LayerNorm Semi-Conjugacy]\label{lem:L_epsilon inversion lemma}
    Let $\varepsilon > 0$ and define $L_\varepsilon:\mathbb{R}^d \to \mathbb{R}^d$ by
    \begin{equation}
        L_\varepsilon(x) \;=\; \frac{x - \mu(x) \mathbf{1}}{\sigma_\varepsilon(x)},
    \end{equation}
    where $\mu(x) = \tfrac{1}{d}\sum_{i=1}^d x_i$ and $\sigma_\varepsilon(x) = \sqrt{\tfrac{1}{d}\sum_{i=1}^d (x_i - \mu(x))^2 + \varepsilon}$.  
    Let $A \in GL(d)$ be fixed. Then there exists a function $f:\mathbb{R}^d \to \mathbb{R}^d$ and a diagonal matrix $D \in  \text{Mat}(d,d)$ such that
    \begin{equation}\label{eq:l_eps_pass}
        L_{\varepsilon}(f(x)) \;=\; (DA)\, L_{\varepsilon}(x), \quad \forall x \in \mathbb{R}^d.
    \end{equation}
Moreover, for any real function (i.e. scalar field) $h: \mathbb{R}^d \to \mathbb{R}$, $f+h\mathbf{1}$ also satisfies the equation \ref{eq:l_eps_pass}.
\end{lemma}

\begin{proof}
Let $H = \{z \in \mathbb{R}^d \mid \mathbf{1}^tz = 0\}$ be the zero-mean subspace. The normalization function $L_\varepsilon$ maps every vector $x \in \mathbb{R}^d$ to a vector in $H$ since $\mu(L_\varepsilon(x)) = 0$.

\textbf{Step 1: Constructing the Matrix $D$}

First, we require the linear transformation $M = DA$ to preserve the zero-mean subspace $H$, i.e., $M(H) \subset H$. Let
\[
\Tilde{v} = (A^t)^{-1}\mathbf{1}, \quad v = \frac{\Tilde{v}}{\|\Tilde{v} \|}.
\]
We define the diagonal matrix $D_{\lambda} = \lambda \cdot \operatorname{Diag}(v)$. For any $z \in H$, we check the mean of $D_\lambda A z$:
\[
\mathbf{1}^t(D_\lambda A)z = (\mathbf{1}^t D_\lambda)(A z) = \lambda v^t A z = \lambda (A^t v)^t z = \frac{\lambda} {\|\Tilde{v}\|} (\underbrace{A^t \Tilde{v}}_{=\mathbf{1}})^t z = \frac{\lambda}{\|\Tilde{v}\|} \mathbf{1}^tz = 0.
\]
Thus, $(D_\lambda A)(H) \subset H$ for any $\lambda$.

Next, we ensure the matrix $M = D_\lambda A$ maps the domain of $L_{\varepsilon}$ on $H$ back into itself. Let $Q \in \text{Mat}(d, d-1)$ be an orthonormal basis matrix for $H$. The operator norm of the restriction of $M$ to $H$ is $\left\|M_{|H} \right\|_2 = \| Q^\top D_\lambda A Q \|_2$. We choose $\lambda_0$ such that this norm is exactly 1:
\[
\lambda_0 = \left\|Q^\top \operatorname{Diag}(v) A Q\right\|_2^{-1}.
\]
Let $M_0 = D_{\lambda_0}A$. With this choice, $M_0$ is a $\lambda$-scaled affine transformation that preserves $H$ and satisfies $\left\|M_{0|H}\right\|_2 = 1$.

\textbf{Step 2: Defining the Domain and Unique Inverse}

Let $\sigma_0(x) = \sqrt{\tfrac{1}{d}\sum_{i=1}^d (x_i - \mu(x))^2}$ be the standard deviation. We have:
\[
\|L_\varepsilon(x)\| = \frac{\|x - \mu(x)\mathbf{1}\|}{\sigma_\varepsilon(x)} = \frac{\sigma_0(x)\sqrt{d}}{\sqrt{\sigma_0(x)^2+\varepsilon}}.
\]
Since $\frac{\sigma_0(x)^2}{\sigma_0(x)^2+\varepsilon} < 1$, the image of $L_\varepsilon$ is the open ball:
\[
B = \operatorname{Im}(L_{\varepsilon}) = \{z \in H \mid \|z\| < \sqrt{d}\}.
\]
Since $\left\|M_{0|H}\right\|_2 = 1$ and $M_0(H) \subset H$, we have $M_0(B) \subset B$. The set $B$ is the domain of the composition $M_0 L_{\varepsilon}(\cdot)$.

For any $z \in B$, the set of pre-images $L_{\varepsilon}^{-1}(z) = \{x \in \mathbb{R}^d \mid L_{\varepsilon}(x) = z\}$ is an equivalence class defined by the mean:
\[
L_{\varepsilon}^{-1}(z) = \{ c + \gamma \mathbf{1} \mid \gamma \in \mathbb{R} \},
\]
where $c$ is the unique vector satisfying $L_{\varepsilon}(c)=z$ and $\mu(c)=0$. We define the unique inverse function $L_{\varepsilon}^{-1}: B \to H$ by choosing the representative with zero mean (i.e., we set $\mu(c)=0$ or $c \in H$).

Solving $z = L_\varepsilon(c) = \frac{c}{\sqrt{\|c\|^2/d + \varepsilon}}$ for $c$ (where $c \in H$, so $\mu(c)=0$):
\[
\|z\| = \frac{\|c\|}{\sqrt{\|c\|^2/d + \varepsilon}}.
\]
Let $r = \|z\|$ and $s = \|c\|$. Solving $r = \frac{s\sqrt{d}}{\sqrt{s^2+\varepsilon}}$ for $s$ yields $s = \frac{r\sqrt{\varepsilon}}{\sqrt{d - r^2 }}$.
The unique zero-mean pre-image is $c = \frac{s}{r} z$. Thus, we define the unique inverse on $H$:
\[
L_{\varepsilon}^{-1}(z) = \sqrt{\frac{d\varepsilon}{d-\|z\|^2}} z. \quad (\text{Note: } L_{\varepsilon}^{-1}(z) \in H)
\]

\textbf{Step 3: Defining the Function $f$}

We start by explicitly defining $f:\mathbb{R}^d \to H$ and then we extend it to a class of functions $\tilde{f} \subset \left\{\tilde{f}: \mathbb{R}^d \to  \mathbb{R}^d\right\}.$ In the case of $H$ our function is given by: $$f:x \in H \mapsto L_{\varepsilon}^{-1}\left( M_0 L_{\varepsilon}(x) \right).$$
Since $L_\varepsilon(x)$ is invariant to the mean, $L_\varepsilon(x) = L_\varepsilon \left(x - c\mathbf{1}\right)$ for any scalar constant $c$.

We can therefore define $\tilde{f}$ as the class of $\mathbb{R}^d \to \mathbb{R}^d$ functions:
\[
\tilde{f} = \{L_{\varepsilon}^{-1}\big( M_0 L_{\varepsilon}(.) \big) + h(.)\mathbf{1}| h \text{ is a } \mathbb{R}^d \to \mathbb{R} \text{ function}\}
\]
Here, $L_{\varepsilon}(x), M_0 L_{\varepsilon}(x) \in B \subset H\hookrightarrow \mathbb{R}^d$ and $L_{\varepsilon}^{-1}(\dots) \in H$ (zero mean) and we remark that $h$ need not be continuous or even measurable.

Finally, we verify the identity:
\begin{align*}
L_{\varepsilon}(f(x)) &= L_{\varepsilon}\Big( L_{\varepsilon}^{-1}\big( M_0 L_{\varepsilon}(x) \big) + h(x)\mathbf{1} \Big) \\
&= L_{\varepsilon}\Big( L_{\varepsilon}^{-1}\big( M_0 L_{\varepsilon}(x) \big) \Big) \quad (\text{Since } L_{\varepsilon}(\cdot) \text{ ignores the mean}) \\
&= M_0 L_{\varepsilon}(x) \\
&= (DA)\, L_{\varepsilon}(x).
\end{align*}
This holds for all $x \in \mathbb{R}^d$.
\end{proof}

\subsubsection{An informal discussion on the approximate linearity
of $f_M$ in high dimensions}\label{rem:linearity_of_f}
While Lemma~\ref{lem:L_epsilon inversion lemma} relies on a carefully
scaled matrix $M_0$ to establish an algebraic semi-conjugacy, it is
worth examining how
$f_M(x) = L_{\varepsilon}^{-1}(M L_{\varepsilon}(x))$ behaves for a
generic matrix $M$ in high dimensions.

\textbf{SVD reduction.}\;
On the zero-mean subspace $H$, the normalization $L_\varepsilon$ is a
radial map, i.e., of the form $x \mapsto \phi(\|x\|) \cdot x$ for a
scalar function $\phi$. Concretely,
$\phi(r) = (r^2/d + \varepsilon)^{-1/2}$. Since the scaling depends
only on $\|x\|$, it commutes with any orthogonal transformation that
preserves $H$: $L_\varepsilon(Rx) = R \cdot L_\varepsilon(x)$, and
likewise for $L_\varepsilon^{-1}$.

This means that if $M = U\Sigma V^T$ is the SVD, the orthogonal
factors pass through the nonlinearity exactly:
\[
    f_M(x) = L_\varepsilon^{-1}\Big(U \Sigma V^T
               L_\varepsilon(x)\Big)
            = U \cdot f_\Sigma(V^T x).
\]
The rotations $U$ and $V^T$ contribute nothing to the nonlinearity
and, in the reparametrization context, can be absorbed into
surrounding weight matrices. The entire nonlinear content of $f_M$
reduces to the diagonal map $f_\Sigma$. Writing $y = V^T x$:
\begin{equation}
    f_\Sigma(y) = \alpha(y) \cdot \Sigma y,
    \qquad
    \alpha(y) = \left(
      1 + \frac{\Delta(y)}{d\varepsilon}
    \right)^{\!-1/2}\!,
\end{equation}
where
$\Delta(y) = \|y\|^2 - \|\Sigma y\|^2 = \sum_i(1-\sigma_i^2)y_i^2$
is the metric defect. The scalar $\alpha$ equals $1$ exactly when
$\Delta \equiv 0$, i.e., when $\Sigma$ is an isometry.

\textbf{Linearization at the origin.}\;
Note that $f_\Sigma(0) = 0$ and, by the chain rule,
\[
  Df_\Sigma(0)
    = DL_\varepsilon^{-1}(0) \cdot \Sigma \cdot DL_\varepsilon(0)
    = \sqrt{\varepsilon}\;\Sigma\;\varepsilon^{-1/2}
    = \Sigma,
\]
so $f_\Sigma(y) = \Sigma y + \mathcal{O}(\|y\|^2)$ near the origin.
The reciprocal scalings of $L_\varepsilon$ and $L_\varepsilon^{-1}$
cancel exactly. Since $f_\Sigma(y) = \alpha(y) \cdot \Sigma y$, the
quality of the linear approximation is controlled entirely by the
deviation $|1 - \alpha(y)|$, which we now estimate.

\textbf{Deterministic regimes.}\;
The deviation of $\alpha$ from $1$ scales as
$|1-\alpha(y)| = \mathcal{O}(|\Delta(y)|/(d\varepsilon))$,
and its behavior depends on the size of $\|y\|$ relative to
$d\varepsilon$. In what follows, we assume all singular values
satisfy $\sigma_i \le 1$, so that $\Delta(y) \ge 0$ for all $y$,
the map $f_\Sigma$ is globally well-defined, and
$\alpha(y) \in (0, 1]$. This assumption is justified below in the
stochastic setting.

\emph{Small inputs} ($\|y\|^2 \ll d\varepsilon$).\;
Since $\Delta(y) \le \max_i(1-\sigma_i^2)\,\|y\|^2$, the ratio
$\Delta(y)/(d\varepsilon)$ vanishes regardless of the spectrum of
$M$. The $\varepsilon$-smoothing alone forces $\alpha \approx 1$,
extending the exact linearization at the origin to a ball of radius
$\mathcal{O}(\sqrt{d\varepsilon})$, with no assumptions on $M$
beyond contractiveness.

\emph{Large inputs} ($\|y\|^2 \gg d\varepsilon$).\;
The smoothing no longer dominates, and $\alpha(y)$ decays as
$\mathcal{O}(\|y\|^{-1})$. The output $f_\Sigma(y)$ saturates in
norm and approaches
$c(\hat{y}) \cdot \Sigma\hat{y}$, where $\hat{y} = y/\|y\|$ and
$c(\hat{y}) =
\sqrt{d\varepsilon / \sum_i(1-\sigma_i^2)\hat{y}_i^2}$
is a direction-dependent constant. The linearization breaks down:
$f_\Sigma$ preserves the directional action of $\Sigma$ but
compresses all inputs to a bounded shell of radius
$\mathcal{O}(\sqrt{d\varepsilon})$. Recovery of approximate
linearity in this regime requires $\Sigma$ to be close to an
isometry: $\max_i(1-\sigma_i^2) \ll d\varepsilon/\|y\|^2$.

\textbf{The stochastic picture.}\;
For a random matrix $M$ with i.i.d.\
$\mathcal{N}(0, 1/d)$ entries, the squared singular values
$\sigma_i^2$ are eigenvalues of $M^TM$ and follow the
Marchenko--Pastur distribution, spreading over $[0, 4]$. Some
singular values exceed $1$; however, by standard results the largest
satisfies $\sigma_{\max} = 1 + \mathcal{O}(d^{-1/2})$. We may
therefore replace $M$ with
$M' = U\min(\Sigma, I)V^T$, clamping each singular value to
$[0,1]$. This modification is necessary: without it, directions
where $\sigma_i > 1$ can drive $\Delta(y)$ negative, producing a
singularity when $|\Delta(y)| = d\varepsilon$. The cost of clamping
is small: $\|M - M'\|_2 = \mathcal{O}(d^{-1/2})$, since the
largest singular value exceeds $1$ by only
$\mathcal{O}(d^{-1/2})$. After clamping, all $\sigma_i \le 1$, the
defect $\Delta(y) \ge 0$ everywhere, and the deterministic analysis
above applies.

The metric defect of the clamped matrix then self-averages. For a
random matrix with i.i.d.\ Gaussian entries, the rotated input
$y = V^Tx$ is uniformly distributed on the sphere of radius
$\|x\|$. The metric defect
$\Delta(y) = \sum_i(1-\sigma_i^2)y_i^2$ is a Lipschitz function on
this sphere, and by linearity of expectation its mean is
$\mathbb{E}[\Delta(y)]
= (\|y\|^2/d)\sum_i(1-\sigma_i^2)
= \|y\|^2(1 - \operatorname{Tr}(\Sigma^2)/d)$.
For square matrices at variance $1/d$,
$\operatorname{Tr}(\Sigma^2)/d \approx 1$, so the mean defect is
small. By standard concentration on the
sphere~\cite{Vershynin_2018}, $\Delta(y)$ deviates from this mean
by at most $\mathcal{O}(d^{-1/2}\|y\|^2)$ with high probability,
giving
\[
  |1 - \alpha(y)|
    \;=\; \mathcal{O}\!\left(
          \frac{\|y\|^2}{d^{3/2}\varepsilon}\right).
\]
Concentration thus extends the linearization well beyond the
$\mathcal{O}(\sqrt{d\varepsilon})$ ball guaranteed by the
deterministic small-input analysis: even for large $\|y\|^2$, the
uniform distribution of $y$ prevents systematic alignment with the
most contractive singular directions, and the self-averaging
suppresses the directional distortion from the large-input regime.
Provided $\varepsilon \gg \|y\|^2 d^{-3/2}$, the scalar $\alpha$
remains close to $1$, and $f_M(x) \approx Mx$ with the
approximation improving as $d$ grows.

\begin{theorem}[Basis Transformation Through LayerNorm]
\label{cor:layernorm-basis}
Let $\Theta \in GL(d)$ be a basis transformation, $MLP: \mathbb{R}^d \to \mathbb{R}^d$ a multilayer perceptron, and $D \in  \text{Mat}(d,d)$ a diagonal matrix with positive entries representing the learned LayerNorm scaling.

If there exist a modified MLP $MLP': \mathbb{R}^d \to \mathbb{R}^d$ and diagonal matrix $D' \in  \text{Mat}(d,d)$ such that
\begin{equation}
    D' L_{\varepsilon}(x + MLP'(x)) = \Theta D L_{\varepsilon}(x + MLP(x)) \quad \forall x \in \mathbb{R}^d,
\end{equation}
then by Lemma~\ref{lem:L_epsilon inversion lemma}, $D'$ and $MLP'$ must satisfy:
\begin{equation}
\begin{aligned}
    D' &= \tilde{D}^{-1} \\
    MLP'(x) &= L_{\varepsilon}^{-1}\Big(M_0 \, L_{\varepsilon}(x + MLP(x))\Big) - x +h(x)\mathbf{1}
\end{aligned}
\end{equation}
where $\tilde{D} = \text{Diag}(v)\lambda_0$, $M_0 = \text{Diag}(v)\lambda_0 \cdot \Theta D$, and
\begin{equation}
    L_{\varepsilon}^{-1}(z) = \sqrt{\frac{d\varepsilon}{d - \|z\|^2}} \, z,
\end{equation}
with $v$ and $\lambda_0$ as constructed in Lemma~\ref{lem:L_epsilon inversion lemma} for $A = \Theta D$ and $h :\mathbb{R}^d\to \mathbb{R}$ is any real function. 
\end{theorem}
Note that $h$ actually adds a single degree of freedom to the $MLP'$ and can be used to pass on information to the skip connection.

\begin{remark}
    The argument above can be distilled down to the case of RMSNorm \cite{zhang2019root}, which has become the normalization of choice in many modern LLMs \cite{Dubey2024TheL3,DA24,qwen2025}. Furthermore, a similar reasoning applies to the recently proposed the Dynamic Tanh alternative to normalization \cite{zhu2025transformersDyT}. Indeed, consider the map

\[
        DyT_{\gamma,\alpha,\beta} : \mathbb{R}^d \to \mathbb{R}^d, \quad 
        (x_1,\ldots,x_d) \mapsto \operatorname{diag}(\gamma)\big(\tanh(\alpha_1 x_1),\ldots,\tanh(\alpha_d x_d)\big) + \beta.
    \]

    This transformation is bijective onto its image whenever $\gamma_i \neq 0$ and $\alpha_i \neq 0$ for all $i$, a condition that is reasonable due to the stochasticity of the optimization. Hence, the structural result established above extends beyond LayerNorm to both RMSNorm and Dynamic Tanh.
\end{remark}

\begin{remark}
This corollary reveals that preserving a basis transformation through LayerNorm requires $MLP'$ to approximate a nonlinear function involving the composition of normalization, the original MLP, and denormalization. Since standard MLPs of the form $W_2\sigma(W_1 x)$ cannot exactly represent this composition, exact Query weight elimination with LayerNorm is not achievable using standard architectural components. This theoretical obstruction motivates either removing LayerNorm or accepting approximate equivalence with adjusted hyperparameters, as we pursue in our experiments. Notably, we test both a slightly larger hidden dimension in the MLP, as well as reduced weight decay compared to the original model.
\end{remark}

\subsection{MLP Basis Transfer Approximation}\label{app:mlp}

We continue with the column-vector convention (see Section~\ref{subsection:Layernorm_n_basischange}).

Standard transformers include skip connections around both attention and MLP blocks. With MLP skip connections, the basis transformation leads to the functional equation:
\begin{equation}
\Omega_1 \circ (\text{MLP} + \text{Id}) \circ \Omega_2^{-1} \approx \text{MLP}' + \text{Id}
\label{eq:mlp-composition}
\end{equation}
where $\Omega_i$ are the basis change matrices. We refer the reader to Theorem~\ref{prop:projector-characterization} for the case of exact solutions with ReLU activation. Here we test whether gradient descent can discover approximate solutions with GELU activations.

\textbf{Experimental setup.} We generate synthetic target functions that mirror the structure arising from basis transformations through skip-connected blocks:
\begin{equation}
Y_{\text{target}}(X) = W_2 \cdot \text{GELU}(W_1 \cdot X) + Z \cdot X - X
\end{equation}
where $W_1 \in \text{Mat}(4h, h)$, $W_2 \in \text{Mat}(h, 4h)$, and $Z \in  \text{Mat}(h,h)$ are randomly initialized and normalized by $1/\sqrt{4h}$, $1/\sqrt{4h}$, and $1/\sqrt{h}$ respectively.

We train a two-layer GELU MLP of the same shape, with a skip connection to approximate this target:
\begin{equation}
Y_{\text{model}}(X) = W_2' \cdot \text{GELU}(W_1' \cdot X) + X
\end{equation}
minimizing the L2 relative squared error objective. We train for 20,000 steps with batch size 65,536 using AdamW (learning rate $5 \times 10^{-3}$ with cosine decay, weight decay $10^{-1}$, gradient clipping at norm 3.0). We test across embedding dimensions $h \in \{256, 512, 768\}$ and compare against an optimal linear baseline matrix $A^*$ fitted via streaming ridge regression (Weight Decay $10^{-6}$) on 500,000 samples.

\textbf{Results.} Figure~\ref{fig:mlp-approx} shows the trained MLP achieves 4 to 6\% relative L2 error across all dimensions, substantially outperforming the optimal linear baseline (11 to 14\%). Mean cosine similarity reaches 0.999 ($\approx 2.5^\circ$ angular error). 

\begin{figure}[htbp]
\centering
\includegraphics[width=\textwidth]{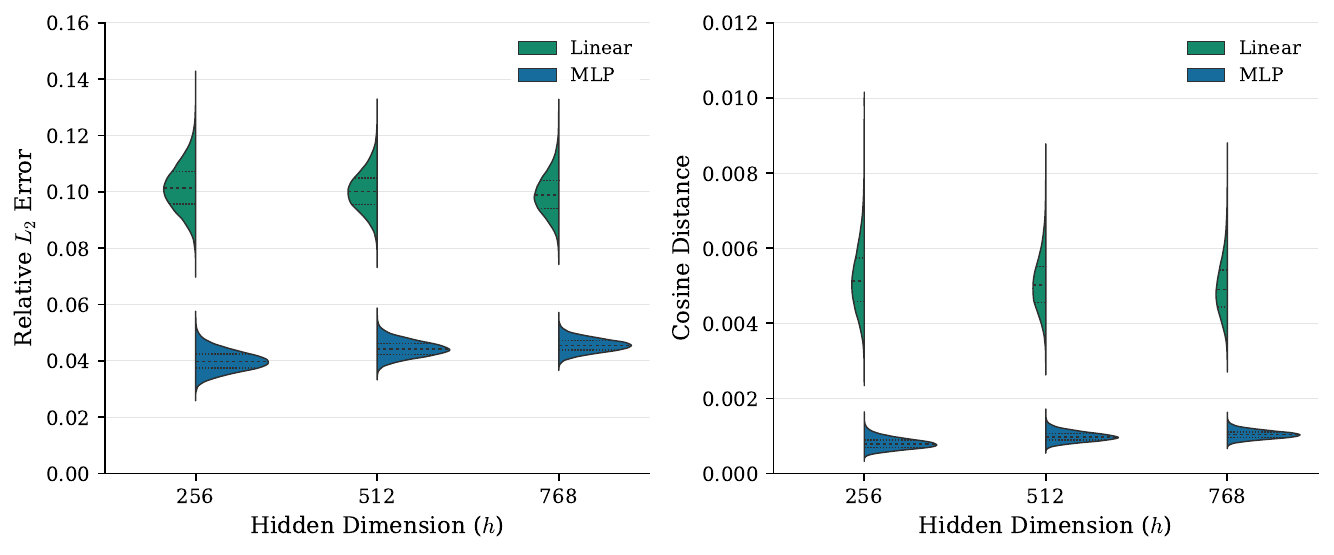}
\caption{\textbf{Left:} Relative $L_2$ error for approximating basis-transformed skip-connected MLPs as a function of hidden dimension $h \in \{256, 512, 768\}$. The trained GELU MLP achieves relative error between $4\%$ and $6\%$, significantly lower than the $11\%$ to $14\%$ error of the optimal linear baseline. \textbf{Right:} Mean cosine distance between predicted and target outputs. The MLP maintains a distance of approximately $0.002$ across all dimensions, while the linear baseline remains above $0.009$, indicating superior directional alignment.}
\label{fig:mlp-approx}
\end{figure}

\textbf{Implications for Query weight elimination.} These results demonstrate that:

\begin{enumerate}
\item \textbf{Approximate solutions are learnable}: While our theory characterizes exact solutions only, standard gradient descent reliably discovers high-quality approximate solutions to the composition problem with GELU activations.

\item \textbf{Nonlinearity is essential}: The trained nonlinear MLP outperforms the optimal linear baseline by a factor of $2\times$ in relative error, confirming that the composition of basis changes with skip connections cannot be adequately approximated by linear transformations alone.

\item \textbf{Dimension-independent quality}: Consistent performance across $h \in \{256, 512, 768\}$ suggests the approximation quality does not degrade with dimension, supporting viability at larger scales.
\end{enumerate}

These findings bridge the gap between our theoretical guarantees (exact solutions, ReLU, restricted skip connections) and practical architectures (approximate solutions, GELU, full skip connections). The key insight is that MLPs can implicitly learn the necessary basis adaptations through gradient-based optimization, even when those adaptations do not admit closed-form solutions.

\subsection{When Can Skip Connections Be Absorbed into ReLU MLPs?}\label{subsec:relu-skip}

We continue with the column-vector convention (see Section~\ref{subsection:Layernorm_n_basischange}).

The theorem below provides a precise characterization of when a skip connection appended to a single-hidden-layer ReLU MLP can be exactly absorbed into an equivalent residual-free ReLU mapping. This question connects two complementary lines of inquiry: the algebraic and reparametrization perspective on identity and residual blocks \cite{He2016}, and geometric analyses of ReLU networks that relate equality of piecewise-linear maps to alignment of their kink hyperplanes and neuron parameters \cite{Arora2018,Montufar2014}.

\begin{theorem}[Residual-free representability of skip-ReLU-MLPs]\label{prop:projector-characterization}
Let \(h\ge 2\) and set \(m=4h\). Let
\[
W_1\in \text{Mat}(m,h),\qquad W_2\in \text{Mat}(h,m)
\]
be given matrices with \(\rank W_1=\rank W_2=h\), where $W_1$ consists of pairwise non-collinear, non-zero rows, and $W_2$ consists of non-zero columns. Denote \(\ReLU\) coordinatewise. There exist matrices
\[
V_1\in \text{Mat}(m,h),\qquad V_2\in \text{Mat}(h,m)
\]
such that the identity
\[
W_2\,\ReLU(W_1 x) + x \;=\; V_2\,\ReLU(V_1 x)
\qquad\text{holds for all }x\in\mathbb{R}^h
\]
if and only if there exists an index set \(J\subset\{1,\dots,m\}\) with \(|J|\ge h\) for which
\[
W_2[:,J]\,W_1[J,:]=-I_h.
\]
\end{theorem}

\begin{proof}
We prove both directions.

\emph{(Sufficiency).} Suppose \(J\subset\{1,\dots,m\}\) satisfies \(|J|\ge h\) and
\[
W_2[:,J]\,W_1[J,:]=-I_h.
\]
Let \(\Pi=\diag(\mathbf{1}_J)\in\{0,1\}^{m\times m}\) be the coordinate projector onto \(J\), and set the sign diagonal
\(D:=I-2\Pi\in\{\pm1\}^{m\times m}\). Define
\[
V_1:=D W_1,\qquad V_2:=W_2.
\]
For any \(y\in\mathbb{R}^m\) write \(\ReLU(y)=\tfrac12(y+|y|)\). Since \(D\) has diagonal entries \(\pm1\) we have \(|V_1x|=|W_1x|\) for all \(x\). Hence
\begin{align*}
V_2\ReLU(V_1x)
&=\tfrac12 V_2\big(V_1x+|V_1x|\big)
=\tfrac12\big(V_2V_1x + V_2|W_1x|\big),\\
W_2\ReLU(W_1x)+x
&=\tfrac12\big(W_2W_1x + W_2|W_1x| + 2x\big).
\end{align*}
Using \(D=I-2\Pi\) and \(W_2\Pi W_1=-I_h\) we obtain
\[
V_2V_1=W_2 D W_1 = W_2(I-2\Pi)W_1 = W_2W_1 - 2(W_2 \Pi W_1) = W_2W_1 + 2I_h.
\]
Substituting into the displayed expressions yields \(V_2\ReLU(V_1x)=W_2\ReLU(W_1x)+x\) for all \(x\). Thus the condition on \(J\) is sufficient.

\medskip

\emph{(Necessity).} Suppose the identity $W_2\,\ReLU(W_1 x) + x = V_2\,\ReLU(V_1 x)$ holds for all $x \in \mathbb{R}^h$.
Using $\ReLU(y) = \tfrac{1}{2}(y + |y|)$, the identity becomes:
\[
W_2\,\tfrac{1}{2}(W_1 x + |W_1 x|) + x = V_2\,\tfrac{1}{2}(V_1 x + |V_1 x|)
\]
Multiplying by 2 and separating the odd and even terms in $x$ yields:
\[
\big(W_2 W_1 + 2I_h - V_2 V_1\big) x = V_2 |V_1 x| - W_2 |W_1 x|
\]
The left-hand side is an odd function of $x$, while the right-hand side is an even function. This implies both sides must be identically zero, giving two conditions:
\begin{align}
 V_2 V_1 &= W_2 W_1 + 2I_h \tag{N1} \\
 V_2 |V_1 x| &= W_2 |W_1 x| \quad \text{for all } x \in \mathbb{R}^h \tag{N2}
\end{align}
We now analyze condition (N2) to relate the rows of $V_1$ and $W_1$.
The function $f(x) = W_2|W_1 x|$ is non-differentiable exactly on the set $\mathcal{H}_W = \bigcup_{i: (W_2)_{:i} \neq 0} \{ x \mid (W_1)_i x = 0 \}$. Since $W_2$ consists of non-zero columns, every hyperplane corresponding to a row of $W_1$ is present in the singular set. Since $W_1$ consists of pairwise non-collinear rows, these hyperplanes are distinct.
For the identity (N2) to hold, the set of non-differentiable points must be identical: $\mathcal{H}_W = \mathcal{H}_V$. This implies that for every row $w$ of $W_1$, there exists a row $v$ of $V_1$ such that they define the same hyperplane (i.e., they are collinear).
Thus, $V_1$ must be a scaled permutation of $W_1$. We can write $V_1 = D P W_1$ for some diagonal matrix $D = \diag(d_1, \dots, d_m)$ with $d_i \ne 0$ and some permutation matrix $P$.

Substitute $V_1 = D P W_1$ into (N2):
\[
V_2 |D P W_1 x| = W_2 |W_1 x|
\]
Let $D_a = \diag(|d_1|, \dots, |d_m|)$. Since a permutation only reorders elements, $|P y| = P|y|$. The equation becomes:
\[
V_2 D_a P |W_1 x| = W_2 |W_1 x|
\]
Since $h \ge 2$ and the rows of $W_1$ are pairwise non-collinear, the functions $\{|w_i^t x|\}_{i=1}^m$ are linearly independent modulo linear functions. This forces the coefficient matrices to be equal: $V_2 D_a P = W_2$. Solving for $V_2$ gives $V_2 = W_2 P^{-1} D_a^{-1} = W_2 P^t D_a^{-1}$.

Now substitute both $V_1 = D P W_1$ and $V_2 = W_2 P^t D_a^{-1}$ into (N1):
\[
(W_2 P^t D_a^{-1}) (D P W_1) = W_2 W_1 + 2I_h
\]
\[
W_2 P^t (D_a^{-1} D) P W_1 = W_2 W_1 + 2I_h
\]
Let $S = D_a^{-1} D$. This is a diagonal matrix with entries $S_{jj} = d_j/|d_j| = \operatorname{sign}(d_j) \in \{\pm 1\}$. Let $D' = P^t S P$. This similarity transformation simply permutes the diagonal entries of $S$, so $D'$ is also a diagonal matrix with entries $\pm 1$. The equation simplifies to:
\[
W_2 D' W_1 = W_2 W_1 + 2I_h
\]
\[
W_2 (D' - I) W_1 = 2I_h
\]
Let $J \subset \{1, \dots, m\}$ be the index set where the diagonal entries of $D'$ are $-1$. Then $D' - I = -2\Pi$, where $\Pi = \diag(\mathbf{1}_J)$ is the coordinate projector onto $J$. Substituting this gives:
\[
W_2 (-2 \Pi) W_1 = 2I_h \implies W_2 \Pi W_1 = -I_h
\]
The matrix $W_2 \Pi W_1$ is precisely the product $W_2[:,J]\,W_1[J,:]$. Thus, we have $W_2[:,J]\,W_1[J,:] = -I_h$.

Finally, by Sylvester's rank inequality:
\[
h = \rank(W_2[:,J]\,W_1[J,:]) \le \min(\rank(W_2[:,J]), \rank(W_1[J,:]))
\]
Since $\rank(W_1[J,:]) \le |J|$, we must have $|J| \ge h$. This completes the proof of necessity.
\end{proof}

\begin{corollary}[Uniqueness under invertible residual perturbations]
\label{cor:uniqueness-invertible-perturbation}
Let \(h \in \mathbb{N}\) with \(h \geq 2\), and set \(m = 4h\). Denote by
\(\sigma: \mathbb{R}^m \to \mathbb{R}^m\) the coordinatewise ReLU function.

For matrices \(W_1 \in  \text{Mat}(m,h)\), \(W_2 \in  \text{Mat}(h,m)\),
and \(Z \in  \text{Mat}(h,h)\), define the map
\[
\phi_{W_1, W_2, Z}: \mathbb{R}^h \to \mathbb{R}^h, \quad
x \mapsto W_2\,\sigma(W_1 x) + Z x.
\]
Consider the parameter space \(\mathcal{W} :=  \text{Mat}(m,h) \times  \text{Mat}(h,m) \times  \text{Mat}(h,h)\).
For Lebesgue-almost every \((W_1, W_2, Z) \in \mathcal{W}\), the matrices satisfy the conditions of Theorem~\ref{prop:projector-characterization} (pairwise non-collinear non-zero rows for $W_1$, non-zero columns for $W_2$), and the following holds:

If \(Z' \in  \text{Mat}(h,h)\) satisfies \(Z' \neq Z\) and \(Z - Z'\) is invertible,
then for every \((W_1', W_2') \in  \text{Mat}(m,h) \times  \text{Mat}(h,m)\),
\[
\phi_{W_1, W_2, Z} \neq \phi_{W_1', W_2', Z'}.
\]
\end{corollary}

\begin{remark}
Theorem~\ref{prop:projector-characterization} and Corollary~\ref{cor:uniqueness-invertible-perturbation} hold without modification for any $m \geq h$.
\end{remark}

\begin{proof}
Suppose for contradiction that \(\phi_{W_1, W_2, Z} = \phi_{W_1', W_2', Z'}\). Then for all \(x \in \mathbb{R}^h\):
\[
W_2 \sigma(W_1 x) + Z x = W_2' \sigma(W_1' x) + Z' x.
\]
Rearranging:
\[
W_2 \sigma(W_1 x) + (Z - Z') x = W_2' \sigma(W_1' x).
\]
Since \(Z - Z'\) is invertible, substituting \(y = (Z - Z')x\) yields:
\[
W_2 \sigma(W_1 (Z - Z')^{-1} y) + y = W_2' \sigma(W_1' (Z - Z')^{-1} y).
\]
Setting \(A = W_2\) and \(B = W_1 (Z - Z')^{-1}\), the left-hand side has the form \(A \sigma(B y) + y\). By the necessity direction of Theorem~\ref{prop:projector-characterization}, there must exist \(J\) with \(|J| \ge h\) such that \(A[:, J] B[J, :] = -I_h\). Substituting back and right-multiplying by \(Z - Z'\):
\[
W_2[:, J] W_1[J, :] = Z' - Z.
\]
For any fixed \(J\), this imposes \(h^2\) algebraic constraints on \((W_1, W_2, Z)\). Under continuous distributions, this set has Lebesgue measure zero. Since there are finitely many choices of \(J\), the union over all \(J\) also has measure zero, so for almost every \((W_1, W_2, Z)\) no such \(J\) exists.
\end{proof}

The condition $W_2[:,J] W_1[J,:] = -I_h$ in Theorem~\ref{prop:projector-characterization} is \emph{non-generic}: it fails with probability 1 under continuous weight distributions. This reveals that ReLU MLPs and Skip-ReLU-MLPs of the same width represent \emph{generically disjoint} function classes. The algebraic condition characterizes exactly when these classes intersect. This clarifies the role of skip connections: they do not provide strictly ``more'' expressivity, but rather access to a \emph{different} region of function space.

\subsection{Proofs from Sections 3 and 4}
\label{app:proofs}

\begin{proof}[Proof of Theorem~\ref{thm:free-lunch} (Single-Layer Query Weight Elimination)]
Set $\Theta = W_Q^j$. For block $i$ with skip connections around both attention and MLP:
\[
B_i(X) = \text{MLP}_i(X + \text{MHA}_i(X)) + X + \text{MHA}_i(X)
\]

\textbf{Inductive hypothesis $\mathcal{H}(i)$:} The output of layer $i$ in the modified model is $\widetilde{X}_i = X_i\Theta$, where $X_i$ is the original output.

\textbf{Base case $\mathcal{H}(0)$:} Set $\widetilde{E} = E\Theta$ and $\widetilde{E}_P = E_P\Theta$. Then $\widetilde{X}_0 = X_0\Theta$, so $\mathcal{H}(0)$ holds.

\textbf{Inductive step $\mathcal{H}(i-1) \Rightarrow \mathcal{H}(i)$:} Assume $\widetilde{X}_{i-1} = X_{i-1}\Theta$. Define reduced weights:
\begin{align*}
\widetilde{W}_Q^i &= \Theta^{-1}W_Q^i, \quad \widetilde{W}_K^i = \Theta^{-1}W_K^i, \quad \widetilde{W}_V^i = \Theta^{-1}W_V^i, \quad \widetilde{W}_O^i = W_O^i\Theta, \\
\widetilde{W}_{up}^i &= \Theta^{-1}W_{up}^i, \quad \widetilde{W}_{down}^i = W_{down}^i\Theta.
\end{align*}

For attention: The projections with input $X_{i-1}\Theta$ are:
\begin{align*}
\text{Query:} \quad (X_{i-1}\Theta)(\Theta^{-1}W_Q^i) &= X_{i-1} W_Q^i, \\
\text{Key:} \quad (X_{i-1}\Theta)(\Theta^{-1}W_K^i) &= X_{i-1} W_K^i, \\
\text{Value:} \quad (X_{i-1}\Theta)(\Theta^{-1}W_V^i) &= X_{i-1} W_V^i.
\end{align*}
Thus $\mathcal{S}_c(X_{i-1}\Theta, \widetilde{W}_Q^i, \widetilde{W}_K^i, \widetilde{W}_V^i) = \mathcal{S}_c(X_{i-1}, W_Q^i, W_K^i, W_V^i)$, and the attention output is:
\begin{align*}
\widetilde{\text{MHA}}_i(X_{i-1}\Theta) &= \mathcal{S}_c(X_{i-1}\Theta, \widetilde{W}_Q^i, \widetilde{W}_K^i, \widetilde{W}_V^i) \cdot \widetilde{W}_O^i \\
&= \mathcal{S}_c(X_{i-1}, W_Q^i, W_K^i, W_V^i) \cdot W_O^i\Theta \\
&= \text{MHA}_i(X_{i-1}) \cdot \Theta.
\end{align*}

After the attention skip connection:
\[
\widetilde{Y}_i = X_{i-1}\Theta + \widetilde{\text{MHA}}_i(X_{i-1}\Theta) = (X_{i-1} + \text{MHA}_i(X_{i-1}))\Theta = Y_i\Theta.
\]

For MLP with skip: Let $Z_i = X_{i-1} + \text{MHA}_i(X_{i-1}) = Y_i$. Then:
\begin{align*}
\widetilde{\text{MLP}}_i(Z_i\Theta) + Z_i\Theta &= \varphi(Z_i\Theta \cdot \Theta^{-1}W_{up}^i)W_{down}^i\Theta + Z_i\Theta \\
&= \varphi(Z_iW_{up}^i)W_{down}^i\Theta + Z_i\Theta = (\text{MLP}_i(Z_i) + Z_i)\Theta.
\end{align*}

After the MLP skip connection:
\[
\widetilde{X}_i = \widetilde{Y}_i + \widetilde{\text{MLP}}_i(\widetilde{Y}_i) = Y_i\Theta + \text{MLP}_i(Y_i)\Theta = X_i\Theta.
\]
Thus $\mathcal{H}(i)$ holds.

\textbf{Layer $j$ achieves identity Query.} At layer $j$, we have $\widetilde{W}_Q^j = \Theta^{-1}W_Q^j = (W_Q^j)^{-1}W_Q^j = I_d$.

\textbf{LM head.} Set $\widetilde{W}_{LM} = \Theta^{-1}W_{LM}$. By $\mathcal{H}(L)$, the final output is $X_L\Theta \cdot \Theta^{-1}W_{LM} = X_L W_{LM}$.
\end{proof}

\begin{proof}[Proof of Theorem \ref{thm:AttentionSkipOnlyQWElim}]
We first establish the key structural result that enables the telescoping argument.

\begin{lemmapush}
\begin{lemma}[Attention-Skip-Only Block Intertwining]\label{lem:intertwining}
Let $B_i:  \text{Mat}(l,d) \to  \text{Mat}(l,d)$ be the $i$-th transformer block with weights $(W_Q^i, W_K^i, W_V^i, W_O^i, W_{up}^i, W_{down}^i)$, skip connection only around attention, and no skip around MLP:
\[
B_i(X) = \varphi\Big((X + \mathcal{S}_c(X, W_Q^i, W_K^i, W_V^i)W_O^i) \cdot W_{up}^i\Big) W_{down}^i
\]
where $\varphi$ is an element-wise activation and $W_Q^i \in GL(d)$.

Then for $\Theta_i = W_Q^i$ and any $\Theta_{i+1} \in GL(d)$, there exists a reduced block $\widetilde{B}_i$ with weights $(I_d, \widetilde{W}_K^i, \widetilde{W}_V^i, \widetilde{W}_O^i, \widetilde{W}_{up}^i, \widetilde{W}_{down}^i)$ such that the following \textbf{intertwining relation} holds for all $X \in  \text{Mat}(l,d)$:
\[
B_i(X) = \widetilde{B}_i(X \cdot \Theta_i) \cdot \Theta_{i+1}^{-1}
\]
This relation is captured by the commutative diagram:
\[
\begin{array}{ccc}
X & \xrightarrow{B_i} & B_i(X) \\
\downarrow\scriptstyle{\cdot\, \Theta_i} & & \downarrow\scriptstyle{\cdot\, \Theta_{i+1}} \\
X \cdot \Theta_i & \xrightarrow{\widetilde{B}_i} & \widetilde{B}_i(X \cdot \Theta_i)
\end{array}
\]
\end{lemma}

\begin{proof}
Define the reduced weights:
\begin{align*}
\widetilde{W}_Q^i &= I_d, &\quad \widetilde{W}_O^i &= W_O^i \Theta_i, \\
\widetilde{W}_K^i &= \Theta_i^{-1}W_K^i, &\quad \widetilde{W}_{up}^i &= \Theta_i^{-1}W_{up}^i, \\
\widetilde{W}_V^i &= \Theta_i^{-1}W_V^i, &\quad \widetilde{W}_{down}^i &= W_{down}^i\Theta_{i+1}.
\end{align*}

We verify the intertwining relation. The attention output with input $X \cdot \Theta_i$ is:
\begin{align*}
\widetilde{\text{MHA}}_i(X\Theta_i) &= \mathcal{S}_c(X\Theta_i, I_d, \Theta_i^{-1}W_K^i, \Theta_i^{-1}W_V^i) \cdot W_O^i\Theta_i \\
&= \mathcal{S}_c(X, W_Q^i, W_K^i, W_V^i) \cdot W_O^i\Theta_i \quad \text{(by the Reparametrization Lemma)}\\
&= \text{MHA}_i(X) \cdot \Theta_i.
\end{align*}

The attention-plus-skip output is:
\[
X\Theta_i + \widetilde{\text{MHA}}_i(X\Theta_i) = X\Theta_i + \text{MHA}_i(X)\Theta_i = (X + \text{MHA}_i(X))\Theta_i = Z_i \cdot \Theta_i,
\]
where $Z_i = X + \text{MHA}_i(X)$ denotes the MLP input in the original block.

The reduced MLP output is:
\begin{align*}
\widetilde{\text{MLP}}_i(Z_i\Theta_i) &= \varphi(Z_i\Theta_i \cdot \Theta_i^{-1}W_{up}^i) \cdot W_{down}^i\Theta_{i+1} \\
&= \varphi(Z_i W_{up}^i) \cdot W_{down}^i\Theta_{i+1} \\
&= \text{MLP}_i(Z_i) \cdot \Theta_{i+1} \\
&= B_i(X) \cdot \Theta_{i+1}.
\end{align*}
Therefore $\widetilde{B}_i(X \cdot \Theta_i) = B_i(X) \cdot \Theta_{i+1}$, which gives $B_i(X) = \widetilde{B}_i(X \cdot \Theta_i) \cdot \Theta_{i+1}^{-1}$.
\end{proof}
\end{lemmapush}

We now proceed with the main proof. The complete model consists of an embedding layer, $L$ transformer blocks, and a final LM Head. For a token sequence $x = (x_0, \ldots, x_k)$, let $E[x]$ denote the embedding matrix and $E_P[0:k]$ the positional embeddings. The input to the first block is $X_0 = E[x] + E_P[0:k]$, and we write $X_i = B_i(X_{i-1})$ for the output of block $i$.

Define the basis change matrices: $\Theta_i = W_Q^i$ for $1 \le i \le L$, and
\begin{align*}
\Theta_{L+1} = \begin{cases}
(\Theta_1^t)^{-1} & \text{if the LM Head and embedding weights are tied} \\
I_d & \text{otherwise}
\end{cases}
\end{align*}

\textbf{Inductive hypothesis $\mathcal{H}(i)$:} The output of the reduced model after block $i$ equals $X_i \cdot \Theta_{i+1}$, where $X_i$ is the output of the original model after block $i$.

\textbf{Base case $\mathcal{H}(0)$:} Define $\widetilde{E} = E\Theta_1$ and $\widetilde{E}_P = E_P\Theta_1$. The input to the first block in the reduced model is:
\[
\widetilde{X}_0 = \widetilde{E}[x] + \widetilde{E}_P[0:k] = (E[x] + E_P[0:k])\Theta_1 = X_0 \cdot \Theta_1.
\]
This equals $X_0 \cdot \Theta_1$, so $\mathcal{H}(0)$ holds.

\textbf{Inductive step $\mathcal{H}(i-1) \Rightarrow \mathcal{H}(i)$:} Assume $\mathcal{H}(i-1)$ holds, i.e., the input to block $i$ in the reduced model is $X_{i-1} \cdot \Theta_i$. By the Intertwining Lemma:
\[
\widetilde{B}_i(X_{i-1} \cdot \Theta_i) = B_i(X_{i-1}) \cdot \Theta_{i+1} = X_i \cdot \Theta_{i+1}.
\]
Thus $\mathcal{H}(i)$ holds.

By induction, $\mathcal{H}(L)$ holds: the output after block $L$ in the reduced model is $X_L \cdot \Theta_{L+1}$.

\textbf{LM Head:} In the untied case, $\Theta_{L+1} = I_d$, so the output is $X_L$. Setting $\widetilde{W}_{LM} = W_{LM}$ gives identical logits.

In the tied case, $\widetilde{W}_{LM} = \widetilde{E}^t = \Theta_1^t E^t$, so:
\[
X_L \cdot \Theta_{L+1} \cdot \widetilde{W}_{LM} = X_L \cdot (\Theta_1^t)^{-1} \cdot \Theta_1^t E^t = X_L \cdot E^t = X_L \cdot W_{LM},
\]
yielding identical final logits.
\end{proof}

\begin{proof}[Proof of Theorem \ref{cor:weight-sharing}]

In order to prove the theorem we will rely on the following lemma:
\begin{lemmapush}
\begin{lemma}[Blocks are conjugate to reduced blocks]\label{lem:conjugate blocks}
Let $B:  \text{Mat}(l,d) \to  \text{Mat}(l,d)$ be a transformer block without layer normalization given by 
\[
B(X) = (\text{Id} + \text{MLP}) \circ (\text{Id} + \text{MHA})(X)
\]
where $\text{MHA}$ uses weights $(W_Q, W_K, W_V, W_O)$ with $W_Q \in GL(d)$, and $\text{MLP}(Y) = \varphi(YW_{up})W_{down}$ for an element-wise activation $\varphi$.

Then there exists $\Theta \in GL(d)$ and a reduced block $\widetilde{B}$ using weights $(I_d, \widetilde{W}_K, \widetilde{W}_V, \widetilde{W}_O, \widetilde{W}_{up}, \widetilde{W}_{down})$ such that 
\[
B = \Theta^{-1} \circ \widetilde{B} \circ \Theta.
\]

In other words, for almost every choice of weights (in the Lebesgue sense), a transformer block without normalization is conjugate to a reduced block.
\end{lemma}

\begin{proof}
Set $\Theta = W_Q$, $\widetilde{W}_K = \Theta^{-1}W_K$, $\widetilde{W}_V = \Theta^{-1}W_V$, $\widetilde{W}_O = W_O\Theta$, $\widetilde{W}_{up} = \Theta^{-1}W_{up}$, and $\widetilde{W}_{down} = W_{down}\Theta$.
\end{proof}
\end{lemmapush}
Let us now continue with the proof of the result by considering the new variables as in the proof of the lemma, and let furthermore

\[
\widetilde{E} = E\Theta,\quad \widetilde{E}_P = E_P\Theta,\quad 
\widetilde{W}_{LM} = \Theta^{-1}W_{LM}.
\]

With these parameters, applying the lemma, each block satisfies

\[
B = \Theta^{-1}\circ \widetilde{B}\circ \Theta.
\]

Since all layers share the same block, the $L$-layer composition telescopes:

\[
B^{\circ L}
= (\Theta^{-1}\circ \widetilde{B}\circ \Theta)^{\circ L}
= \Theta^{-1}\circ \left(\widetilde{B}\right)^{\circ L}\circ \Theta.
\]

The embedding and output head reparametrizations complete the proof.
\end{proof}

\section{Acknowledgements}
We would like to thank the anonymous reviewers at the ICLR DeLTA workshop for their helpful comments and suggestions, as well as Nils Graef (OpenMachine), Borjan Geshkovski (INRIA/Sorbonne), François Yvon (CNRS/Sorbonne) and Yiping Ji (University of Adelaide) for helpful discussions on this work. The first author would also like to thank Igor Ševo (HTEC/UniBL) for discussions on the attention mechanism. We also thank Dimitar Peshevski (UKIM) for suggestions on parameter analysis for hyperparameter improvements in future work.

This work was initiated by the first author while they were at the HTEC Group~\cite{karbevski2024preliminary}, and we thank the team for their support.

\addcontentsline{toc}{section}{References} 
\bibliographystyle{alpha}
\bibliography{sample}

\end{document}